\documentclass[10pt,letterpaper]{article}
\usepackage[margin=1in]{geometry}
\usepackage{times}
\usepackage[utf8]{inputenc}
\usepackage[T1]{fontenc}
\usepackage{url}
\usepackage{booktabs}
\usepackage{amsfonts}
\usepackage{amsmath}
\usepackage{amssymb}
\usepackage{amsthm}
\usepackage{nicefrac}
\usepackage{microtype}
\usepackage{xcolor}
\usepackage{graphicx}
\graphicspath{{figures/}{./}}
\usepackage{multirow}
\usepackage{colortbl}
\usepackage{pifont}
\usepackage{algorithm}
\usepackage{algpseudocode}
\definecolor{mylightblue}{rgb}{0.9, 0.9, 1.0}
\usepackage{tikz}
\usetikzlibrary{positioning, arrows.meta, fit, calc}
\usepackage{placeins}
\newif\ifrevisionhighlight
\revisionhighlightfalse
\newcommand{\rev}[1]{#1}
\newenvironment{revision}{}{}
\usepackage{natbib}
\setcitestyle{authoryear,round,citesep={;},aysep={,},yysep={;}}
\usepackage[hidelinks]{hyperref}

\let\arxivabstract\abstract
\renewcommand{\abstract}{\arxivabstract\noindent\ignorespaces}

\newcommand{\tsinghua}{\textsuperscript{1}}
\newcommand{\tencent}{\textsuperscript{2}}

\author{
  \parbox{\dimexpr\linewidth-2\tabcolsep\relax}{\centering
    Chao Wang\tsinghua\thanks{Work done while interning at Tencent.}\quad
    Hongtao Tian\tencent\quad
    Tao Yang\tencent\thanks{Corresponding Author.}\quad
    Yunsheng Shi\tencent\quad
    Ting Yao\tencent\quad
    Wenbo Ding\tsinghua\footnotemark[2]
  }\\
  \parbox{\dimexpr\linewidth-2\tabcolsep\relax}{\centering
    \tsinghua Tsinghua University \quad
    \tencent WeChat, Tencent
  }
}
\hypersetup{pdfauthor={Chao Wang, Hongtao Tian, Tao Yang, Yunsheng Shi, Ting Yao, Wenbo Ding}}

\date{}
\begin{document}
\title{\raggedright Process Advantage Signal Shaping:\\
Advantage Construction for Process-Supervised Reinforcement Learning in LLM Reasoners}

\maketitle

\begin{abstract}
\rev{Dense process supervision provides local evaluations, but leaves open how these scores should become policy advantages. We introduce \textbf{PASS} (\emph{Process Advantage Signal Shaping}), an advantage-construction method for group-relative policy optimization (GRPO). PASS separately normalizes and fuses reward channels, constructs batch-shared credit units from variation in the fused score trace, and normalizes return-to-go by the remaining unit count. These operations connect feedback scale, credit resolution, and accumulation horizon in a single computation. We evaluate learned process reward model (PRM) feedback for mathematical reasoning and teacher-derived token feedback for multi-hop question answering (QA). Across seven mathematics benchmarks, PASS improves average pass@1 by 4.62 percentage points over outcome-only GRPO with approximately $2\%$ shorter responses. Matched comparisons with Step-GRPO and component ablations further support the choice of shared credit units and residual normalization. On multi-hop QA, PASS improves over process-signal controls across two normalizers; token-level ablations reveal an accuracy--length trade-off: token-level accumulation slightly favors pass@1, while PASS achieves higher pass@8/16 with shorter responses. The results establish the conversion from local process scores to policy advantages as a consequential design choice across distinct supervision sources.}
\end{abstract}

\section{Introduction}
\label{sec:introduction}

Process supervision offers feedback within a reasoning trajectory rather than only at its conclusion. Learned process reward models (PRMs) score intermediate steps~\citep{lightman2023lets,zhang2025lessons}, while \rev{on-policy distillation (OPD)} supplies teacher-derived token log-probability ratios~\citep{agarwal2023onpolicy}. These signals can support group-relative policy optimization (GRPO)~\citep{shao2024deepseekmath,guo2025deepseek}. \rev{DeepSeekMath's process-supervised GRPO already constructs token advantages by accumulating group-normalized rewards at the current and subsequent reasoning-step endpoints~\citep{shao2024deepseekmath}. This estimator motivates examining how reward channels, accumulation units, and remaining-horizon scaling shape the resulting credit.} A local score must still be normalized, combined with sequence-level rewards, and transformed into the per-token advantage used to update the policy. The same upstream feedback can therefore lead to different learning signals depending on these choices.

\begin{revision}
The interface between process supervision and GRPO involves three coupled decisions: which reward values enter the update, which units accumulate future feedback, and how cumulative credit scales with the remaining horizon. In particular, the steps used by a PRM to score a response need not be the units best suited to advantage construction. Their boundaries can depend on formatting conventions and differ across rollouts. Token-level teacher feedback presents the same interface at a finer resolution. This motivates studying credit units defined jointly across trajectories and adapted to the variation in the available feedback.

We introduce \textbf{PASS} (\emph{Process Advantage Signal Shaping}) to make these decisions explicit. \emph{Advantage Fusion} (AF) separately normalizes reward channels and fuses their contributions. \emph{Chunk-by-Value} (CV) detects changes in the masked positional mean of the fused traces, creating batch-shared, reward-adaptive accumulation units. Return-to-go (RTG) gathers subsequent feedback over those units, and \emph{Divide-Length} (DL) scales it by a power of their remaining count. AF determines the score geometry, CV defines its computational resolution, and DL adjusts the scale of accumulated credit. The resulting token advantages enter the existing GRPO clipped surrogate and regularizer.
\end{revision}

PASS builds on work on reward-channel normalization~\citep{ye2025correctness,liuGDPOGrouprewardDecoupled2026}, credit-assignment granularity~\citep{zhang2026reasoning}, and length-dependent reward transformation~\citep{razin2025WhatMakesReward,li2026tackling}. \rev{PASS defines batch-shared, reward-adaptive credit units from variation in the fused score trace, then performs RTG accumulation and residual-count normalization over those units. Fusion, boundary selection, accumulation, and scaling jointly determine the token-level advantages.} A fixed-score analysis clarifies the role of the normalization exponent. Detailed comparisons with related approaches appear in Section~\ref{sec:related_work}.

\begin{revision}
Our contribution is a concrete shared, feedback-driven advantage construction and an empirical study of how its design affects learning across two signal sources. On mathematics, the outcome-only comparison establishes overall effectiveness: PASS improves average pass@1 by 4.62 percentage points (pp) with approximately $2\%$ shorter responses. Matched Step-GRPO and native-step comparisons then test the value of the particular construction, while component ablations and training dynamics connect final performance to credit resolution and residual normalization (Section~\ref{sec:exp_math}). On \rev{multi-hop question answering (QA)}, process-signal and normalization controls test transfer to teacher feedback, and token-level ablations characterize the accuracy--length behavior of shared accumulation units (Section~\ref{sec:exp_hotpotqa}). Together, these studies examine how existing supervision can be organized into a more effective policy-learning signal.
\end{revision}

\section{Related Work}
\label{sec:related_work}

\paragraph{Process rewards and mixed-reward GRPO.} Step-level verifiers~\citep{lightman2023lets} and subsequent PRM training work~\citep{zhang2025lessons} provide dense feedback for reasoning. \rev{DeepSeekMath's process-supervised GRPO~\citep{shao2024deepseekmath} normalizes all PRM step rewards across the responses sampled for one question with shared mean and standard deviation, then accumulates the normalized rewards at step endpoints at or after each token. Its credit units follow the native scoring steps.} Studies of implicit process estimation~\citep{sullivan2025grpo}, outcome/process consistency~\citep{ye2025correctness}, and decoupled reward normalization~\citep{liuGDPOGrouprewardDecoupled2026} address complementary aspects of this problem. \rev{PASS studies the channel handling, credit units, and horizon scaling of this advantage construction: AF separately normalizes and fuses feedback channels, CV defines batch-shared units from fused-score variation, and DL normalizes RTG by their remaining count. Our native PRM-step RTG ablation retains the common AF, normalization, and training settings to compare credit units and DL; it is not a reproduction of DeepSeekMath's original training recipe.} \rev{Step-GRPO~\citep{li2026stepgrpo} uses step-attention to combine PRM scores into a trajectory reward before group-relative optimization. We compare Step-GRPO with PASS under a common actor, scorer, and training budget; implementation settings are specified in Appendix~\ref{app:revision_results}.}

\paragraph{Credit-assignment resolution.} Reasoning credit assignment spans token, span, and step granularity~\citep{zhang2026reasoning}. \rev{Future-KL Influenced Policy Optimization (FIPO)}~\citep{ma2026fipo} uses current/old-policy log-probability information, \rev{with KL denoting Kullback--Leibler divergence}, while counterfactual span weighting~\citep{khandoga2026uniform} and \rev{Attribution-based Contribution to Policy Optimization (ACPO)}~\citep{yin2025pinpointing} investigate other ways to localize credit. \rev{Decoupled Reward Policy Optimization (DRPO)}~\citep{li2025drpo} studies the advantage baseline. CV derives numerical endpoints from the supplied score traces, providing accumulation units without an auxiliary semantic segmenter. This emphasis on score-defined resolution complements approaches that use semantic or counterfactual information.

\paragraph{Length and reward transformation.} Prior analyses examine length dependence in reasoning \rev{reinforcement learning (RL)}~\citep{razin2025WhatMakesReward,liu2025enhancing}. GR$^3$~\citep{li2026tackling} rescales raw rewards before group normalization, and post-hoc calibration~\citep{huang2024posthoc} studies reward-related corrections. DL instead normalizes the fused chunk return-to-go by the remaining chunk count. Our GR$^3$-style comparison uses the implementation settings specified in Appendix~\ref{app:hparams}.

\paragraph{On-policy distillation.} On-policy distillation~\citep{agarwal2023onpolicy} and its generalized variants~\citep{yang2026learning} provide teacher-dependent feedback on student trajectories. Analyses of distillation dynamics~\citep{li2026rethinking} motivate attention to where this feedback is applied. We use sampled log-probability-ratio signals to study advantage construction with token-varying feedback, complementing the piecewise-constant PRM setting on mathematics.

\section{Preliminaries}
\label{sec:preliminaries}

GRPO~\citep{shao2024deepseekmath,guo2025deepseek} samples $G$ responses to each prompt and uses relative reward information to construct an advantage. Let $i$ index a response, $t$ a token position, and $M_{i,t}$ its response mask. \rev{We use $M_{i,t}\in\{0,1\}$, with 1 marking a valid process-signal position; subscripts $p,o,f$ denote process, outcome, and format channels.} A process signal $s_{i,t}$ supplies local feedback, while $r_{o,i}$ and $r_{f,i}$ are sequence-level outcome and format rewards. We write $z_{o,i}$ and $z_{f,i}$ for their separately group-standardized values. Unlike these sequence-level channels, a process signal need not be constant along a response.

\paragraph{Normalization scope.} Distinguish a same-prompt group of $G$ responses from the rollout batch $\mathcal B$ passed to advantage construction. Outcome and format statistics are computed within same-prompt groups. The full PASS process-channel normalizer pools valid tokens in $\mathcal B$:
\begin{equation}\label{eq:pre_global_signal_set}
\mathcal V=\{s_{i,t}:i\in\mathcal B,\ M_{i,t}=1\}.
\end{equation}
\rev{Here $\mathcal V$ is a multiset, retaining one entry per valid token; $\mu_{\mathcal V}$ and $\sigma_{\mathcal V}$ are its mean and sample standard deviation, and $\epsilon>0$ is a numerical stabilizer.} The two process-channel operators in QA are
\begin{equation}\label{eq:pre_normalized_step_signal}
z_{p,i,t}=\mathcal N(s_{i,t};\mathcal V)=
\begin{cases}
(s_{i,t}-\mu_{\mathcal V})/(\sigma_{\mathcal V}+\epsilon), & \text{Masked-Norm},\\
s_{i,t}/(\max_{v\in\mathcal V}|v|+\epsilon), & \text{Abs-Max}.
\end{cases}
\end{equation}
The standard deviations in advantage construction use the sample correction. Masks exclude padding from the process pool, while sequence-level rewards are standardized within each same-prompt group. Mathematical reasoning uses Masked-Norm; QA compares both operators.

\paragraph{Alternative estimators.} \rev{Future-step accumulation already appears in DeepSeekMath's process-supervised GRPO~\citep{shao2024deepseekmath}. We write this RTG operation for general selected credit units, so that AF channel fusion, CV unit selection, and DL remaining-count scaling can be specified separately.} \rev{Let $j$ index selected endpoints, $c_{i,j}$ denote the score at endpoint $j$ (defined from the fused token trace $a_{i,t}$ in Section~\ref{subsec:method_cv}), and $R_{i,j}$ denote its return-to-go:}
\begin{revision}
\begin{equation}\label{eq:pre_cumulative_advantage}
R_{i,j}=\sum_{\ell=j}^{K_i}c_{i,\ell},
\end{equation}
\end{revision}
where $K_i$ counts selected endpoints for response $i$. A direct token estimator instead uses $a_{i,t}$ without this sum. The process-only and AF-only QA controls use direct estimators; PASS combines chunking, RTG, and residual-count normalization.

\paragraph{Policy update.} Token advantages $A_{i,t}$ are held fixed while optimizing the GRPO clipped surrogate. Its token contribution is $\min(\rho_{i,t}A_{i,t},\operatorname{clip}(\rho_{i,t},1-\varepsilon_c,1+\varepsilon_c)A_{i,t})$. \rev{Here $\rho_{i,t}$ is the updated-policy to pre-update-policy likelihood ratio at the sampled token, and $\varepsilon_c>0$ is the clipping threshold.} Gradients do not pass through reward extraction, normalization statistics, or segmentation. We retain the masked loss reduction and reference regularization described in Appendix~\ref{app:hparams}.

\section{Method: The PASS Pipeline}
\label{sec:method}

PASS connects three choices in advantage construction: the relative scale of reward channels, the units of credit accumulation, and normalization over the remaining horizon. Figure~\ref{fig:pass_workflow} summarizes the computation. Channel fusion produces a token score trace; CV selects endpoints from that trace; \rev{RTG accumulates future endpoint scores, DL normalizes by the remaining endpoint count, and the result is broadcast to tokens.}

\begin{figure}[ht]
\centering
\resizebox{0.8\textwidth}{!}{%
\begin{tikzpicture}[
    >=Stealth,
    pass_block/.style={draw=blue!60, fill=blue!5, thick, rounded corners, align=center, minimum width=3.2cm, minimum height=1.2cm, font=\small},
    signal_block/.style={draw=gray!80, fill=gray!10, thick, align=center, minimum width=0.75cm, minimum height=0.8cm, font=\footnotesize},
    surrogate/.style={draw=red!60, fill=red!5, thick, rounded corners, align=center, minimum width=3.5cm, minimum height=1.2cm, font=\small}
]
\node[pass_block] (AF) {\textbf{AF}\\Fuse reward channels};
\node[pass_block, right=0.8cm of AF] (CV) {\textbf{CV}\\Define chunks};
\node[pass_block, right=0.8cm of CV] (DL) {\textbf{DL}\\Normalize chunk RTG};
\node[
    draw=blue!30,
    dashed,
    thick,
    inner sep=12pt,
    fit=(AF) (CV) (DL),
    label={[font=\bfseries, text=blue!80!black]above:PASS Pipeline}
] (pass_box) {};
\node[surrogate, below=1.8cm of DL, xshift=0.5cm] (surr)
{\textbf{Clipped Surrogate}\\(GRPO Objective)};
\node[signal_block] (s0) at ([xshift=-0.9cm]AF |- surr.center) {$\rev{s_{i,t}}$};
\node[signal_block, right=0.15cm of s0] (s1) {$\rev{r_{o,i}}$};
\node[right=0.18cm of s1, font=\footnotesize] (dots) {};
\node[signal_block, right=0.18cm of dots] (sn) {$\rev{r_{f,i}}$};
\node[
    draw=gray!70,
    dashed,
    thick,
    rounded corners,
    inner sep=8pt,
    fit=(s0) (s1) (dots) (sn)
] (signals_box) {};
\draw[->, dashed, very thick, gray!80]
    (signals_box.east) --
    node[above, font=\footnotesize, text=gray!90] {Baseline estimator}
    (surr.west);
\draw[->, very thick, blue!80!black, rounded corners]
    (signals_box.north) -- ++(0, 0.6) -| (AF.south);
\draw[->, very thick, blue!80!black] (AF) -- (CV);
\draw[->, very thick, blue!80!black] (CV) -- (DL);
\draw[->, very thick, blue!80!black, rounded corners]
    (DL.south) -- ++(0, -0.6) -| (surr.north);
\end{tikzpicture}%
}
\caption{PASS advantage construction. AF combines separately normalized process, outcome, and format channels; CV defines chunks from the fused trace; DL normalizes the chunk return-to-go. The resulting token advantages enter the GRPO clipped surrogate.}
\label{fig:pass_workflow}
\end{figure}

\subsection{Advantage Fusion}
\label{subsec:method_shaping}

AF makes channel weights explicit after normalizing process, outcome, and format rewards separately. This avoids coupling their relative contributions through a common normalization scale. Let $z_{p,i,t}$ be the normalized process trace and $z_{o,i},z_{f,i}$ the normalized sequence-level rewards from \S\ref{sec:preliminaries}. With nonnegative weights, $w_f>0$, and $W=w_p+w_o+w_f$, the fused score at a valid response position is
\begin{equation}\label{eq:prcvdl_advantage_shaping}
a_{i,t}=\gamma
\begin{cases}
w_p z_{p,i,t}+w_o z_{o,i}+w_f z_{f,i}, & r_{f,i}>0,\\
W z_{f,i}, & r_{f,i}=0.
\end{cases}
\end{equation}
The format reward is binary. For format-compliant responses, all three channels contribute; otherwise, the format channel alone supplies credit. The scale is $\gamma=1/W$ in QA and $\gamma=1$ in mathematics. Thus QA uses a weighted average with a format-only fallback, whereas mathematics uses the corresponding weighted sum. This scale is part of the training recipe because it sets the strength of the advantage relative to fixed regularization. We use unit channel weights for full PASS.

The QA controls are defined directly from these quantities: process-only uses $A_{i,t}=z_{p,i,t}$, AF-only uses $A_{i,t}=a_{i,t}$, and outcome-only uses the AF expression with $w_p=0$. All three use direct token advantages without RTG. In the outcome-only control, $W=2$ rather than 3, so the relative coefficients of the remaining channels change accordingly. Table~\ref{tab:hparams_diff} gives the complete configurations.

\paragraph{Process signals.} For mathematics, a learned PRM assigns a score to each solution step, which is broadcast over the tokens of that step. QA uses the sampled log-probability-ratio signal
\begin{equation}\label{eq:opd_signal}
s_{i,t}^{\rm OPD}=\ell_{\rm teacher}(x_{i,t}\mid h_{i,t})-\ell_{\rm student}(x_{i,t}\mid h_{i,t}),
\end{equation}
\rev{Here $x_{i,t}$ is the sampled token, $h_{i,t}$ contains the prompt and preceding response tokens, and $\ell_{\rm teacher}$ and $\ell_{\rm student}$ are log probabilities under the teacher and rollout student policies, respectively.} The \rev{generalized on-policy distillation (G-OPD)} variant~\citep{yang2026learning} introduces a base-policy anchor:
\begin{equation}\label{eq:prcvdl_gopd_signal}
s_{i,t}^{\rm G\text{-}OPD}=(\ell_{\rm base}-\ell_{\rm student})+
\lambda(\ell_{\rm teacher}-\ell_{\rm base}).
\end{equation}
\rev{The base policy is denoted by $\pi_{\rm base}$, with log probability $\ell_{\rm base}$; $\lambda$ weights the teacher--base difference. All terms are evaluated at the same sampled token and prefix, and $\lambda=1$ recovers OPD.} Both signals can be positive or negative. \rev{We use $D_{\rm KL}$ for the KL divergence.} For OPD, sampling from the same student distribution appearing in the ratio gives conditional expectation $-D_{\rm KL}(\pi_{\rm student}\Vert\pi_{\rm teacher})$; top-$p$ truncation need not preserve this equality. We apply the chosen normalizer to the sampled scores rather than assuming a fixed sign.

\subsection{Chunk-by-Value}
\label{subsec:method_cv}

\rev{Process signals provide local evaluations, but their scoring granularity need not define the units of credit accumulation in GRPO. PRM steps may follow formatting or segmentation conventions specific to each response; teacher-derived feedback varies at individual tokens. CV constructs batch-shared, reward-adaptive credit units from changes in the fused score trace. Shared positional boundaries give rollouts a common computational coordinate, while feedback variation determines its resolution. RTG and DL then use these units for accumulation and residual-horizon normalization.}

\rev{To obtain the shared partition, CV uses the masked positional mean of the fused traces.} Let $[t_0,t_1]$ cover the valid response positions in the batch. For both domains, define
\begin{equation}\label{eq:cv_trace}
u_t=\frac{\sum_{i\in\mathcal B}M_{i,t}a_{i,t}}
{\max(1,\sum_{i\in\mathcal B}M_{i,t})},\qquad t_0\leq t\leq t_1.
\end{equation}
\rev{For $t=t_0+1,\ldots,t_1$, an endpoint is placed at $t-1$ when} $|u_t-u_{t-1}|>\eta$, and the shared interval endpoint $t_1$ is included. For response $i$, only endpoints lying within its valid mask are retained. \rev{Write their set as $B_i$ and its ordered elements as} $b_{i,1}<\cdots<b_{i,K_i}$ and \rev{define the chunk score $c_{i,j}=a_{i,b_{i,j}}$, keeping $a_{i,t}$ for the token-level fused trace.} We use $\eta=0.005$ for the continuously varying QA signal and $\eta=10^{-8}$ for the piecewise-constant PRM signal. \rev{These boundaries define numerical credit units; their relationship to individual score changes and semantic steps is discussed in Section~\ref{app:discussion}.}

\rev{For selected chunk values $v_i=(v_{i,1},\ldots,v_{i,K_i})$, broadcasting assigns}
\begin{equation}\label{eq:prcvdl_cv_advantage_broadcast}
\operatorname{Broadcast}(v_i)_t=v_{i,j},\qquad
b_{i,j-1}<t\leq b_{i,j},\quad b_{i,0}=t_0-1,
\end{equation}
with zero outside the selected intervals. \rev{Only valid shared endpoints are retained; no response-specific terminal endpoint is appended. Tokens beyond the last retained endpoint receive zero direct advantage, and all positions do so if $K_i=0$. Reference regularization remains separate. Appendix~\ref{app:algorithms} specifies masking, and Appendix~\ref{app:revision_dynamics} reports coverage.}

\subsection{Divide-Length}
\label{subsec:method_dl}

DL makes the scale of accumulated credit depend on the number of remaining chunks. \rev{The exponent $k\in(0,1]$ controls this horizon scaling.} For the chunk scores selected by CV, it computes
\begin{revision}
\begin{equation}\label{eq:prcvdl_dl_residual}
A^{\rm DL}_{i,j}=\frac{\sum_{\ell=j}^{K_i}c_{i,\ell}}{(K_i-j+1)^k}.
\end{equation}
\end{revision}
These values are broadcast by Equation~\ref{eq:prcvdl_cv_advantage_broadcast}. At $k=1$ the result is the mean score over the remaining chunks; smaller exponents retain more dependence on the accumulation horizon. We use $k=1$ in the main mathematics experiments and $k=0.7$ in QA. Normalization applies to the fused score, including the outcome and format channels: a constant contribution $C_i$ becomes $C_i n^{1-k}$ with $n$ remaining chunks. Algorithm~\ref{alg:pass_pipeline} gives an end-to-end overview.

\paragraph{A fixed-score diagnostic.} To interpret attenuation separately from training, freeze the chunk scores, existing boundaries, and normalization statistics. \rev{For one response, suppress $i$ and write $K=K_i$, $c_j=c_{i,j}$, and $R=\sum_{j=1}^{K}c_j$. Define}
\begin{equation}\label{eq:prcvdl_dl_path_score}
S=R/K^k.
\end{equation}
Appending one new chunk of score $\delta$ gives
\begin{equation}\label{eq:prcvdl_dl_change}
\Delta S=\frac{R+\delta}{(K+1)^k}-\frac{R}{K^k},
\end{equation}
and hence the exact condition
\begin{equation}\label{eq:prcvdl_dl_marginal}
\Delta S>0\quad\Longleftrightarrow\quad
\delta>R\left[\left(1+\frac1K\right)^k-1\right].
\end{equation}
At $k=1$, an appended chunk raises the diagnostic precisely when its score exceeds the existing average $R/K$. For $k<1$, the threshold reflects weaker length attenuation. This relationship describes a fixed-score diagnostic, rather than a step-acceptance rule or the dynamics of the complete clipped update. In particular, $k<1$ does not guarantee a finite optimal length. The derivation and repeated-contribution behavior are given in Appendix~\ref{app:proofs}.

\section{Experiments}
\label{sec:experiments}

We evaluate how advantage construction affects performance with PRM feedback on mathematics and teacher-derived feedback on multi-hop QA. The experiments compare overall accuracy, the effect of channel handling and normalization, and the accompanying response lengths. Main comparisons use multiple training seeds; additional single-seed studies examine the normalization exponent and PRM scale. Training and evaluation protocols are specified in Appendix~\ref{app:hparams}.

\subsection{Mathematical Reasoning with a PRM Process Signal}
\label{sec:exp_math}

\paragraph{Setup.} We train Qwen2.5-Math-7B~\citep{yang2024qwen25mathtechnicalreportmathematical} on DeepScaleR~\citep{deepscaler2025}, using Qwen2.5-Math-PRM-7B~\citep{zhang2025lessons} for the PASS process channel, and evaluate on AIME24/25, AMC23, GSM8K, MATH, Minerva, and Olympiad. Correctness is checked with \texttt{math\_verify}~\citep{kydlicek2025MathVerifyMathVerification}. We report mean and \rev{sample standard deviation (SD)} across three PASS training runs and two outcome-only GRPO runs (labeled GRPO(ORM) in the tables). The pre-RL base model provides an additional reference. Per-question sampling counts and metric definitions are given in Appendix~\ref{app:hparams}.

\paragraph{Overall performance.} PASS achieves $36.04\pm0.22$ average pass@1 versus $31.41\pm0.35$ for outcome-only GRPO, an improvement of $4.62$ pp (Table~\ref{tab:math_main}). Every PASS run scores above every baseline run on each of the seven benchmarks. Gains are largest on GSM8K ($+6.28$), AMC23 ($+5.83$), and MATH ($+5.33$), covering both school-level and competition mathematics. The pass@8 average also improves by $2.76$ pp, with the largest differences on AIME24 and AIME25 (Appendix~\ref{app:math_sensitivity}).

\begin{table}[htbp]
\centering
\caption{pass@1 (\%) on mathematical reasoning. Deltas are computed before rounding; the average gives each benchmark equal weight.}
\label{tab:math_main}
\begin{tabular}{@{}lrrrr@{}}
\toprule
\textbf{Benchmark} & \textbf{Pre-RL} & \textbf{GRPO(ORM)} & \cellcolor{mylightblue}\textbf{PASS (Ours)} & \textbf{$\Delta$} \\
\midrule
AIME24 & $8.9$ & $12.36 \pm 0.10$ & \cellcolor{mylightblue}$\mathbf{16.95 \pm 0.95}$ & $+4.59$ \\
AIME25 & $2.3$ & $4.48 \pm 0.37$ & \cellcolor{mylightblue}$\mathbf{9.18 \pm 0.37}$ & $+4.70$ \\
AMC23 & $22.8$ & $47.88 \pm 1.81$ & \cellcolor{mylightblue}$\mathbf{53.72 \pm 1.36}$ & $+5.83$ \\
GSM8K & $30.1$ & $71.27 \pm 0.17$ & \cellcolor{mylightblue}$\mathbf{77.56 \pm 1.23}$ & $+6.28$ \\
MATH & $27.9$ & $49.96 \pm 0.18$ & \cellcolor{mylightblue}$\mathbf{55.29 \pm 0.81}$ & $+5.33$ \\
Minerva & $8.4$ & $24.33 \pm 0.15$ & \cellcolor{mylightblue}$\mathbf{27.90 \pm 0.60}$ & $+3.57$ \\
Olympiad & $4.1$ & $9.60 \pm 0.05$ & \cellcolor{mylightblue}$\mathbf{11.65 \pm 0.20}$ & $+2.05$ \\
\midrule
\textbf{Average} & $14.9$ & $31.41 \pm 0.35$ & \cellcolor{mylightblue}$\mathbf{36.04 \pm 0.22}$ & $\mathbf{+4.62}$ \\
\bottomrule
\end{tabular}
\end{table}

\paragraph{Accuracy and response length.} The accuracy gains accompany approximately $2\%$ lower mean response length (Table~\ref{tab:math_length}). Reductions are largest on GSM8K ($-11.9\%$) and Minerva ($-14.4\%$), while four benchmarks show increases of at most $3\%$. PASS therefore improves aggregate accuracy without increasing the aggregate token count on this panel. The more pronounced accuracy--length trade-off in QA is examined below.

\begin{table}[htbp]
\centering
\caption{Average response length in tokens on the mathematical-reasoning panel. PASS's seven-benchmark average is approximately $2\%$ lower.}
\label{tab:math_length}
\begin{tabular}{@{}lrrr@{}}
\toprule
\textbf{Benchmark} & \textbf{GRPO(ORM)} & \cellcolor{mylightblue}\textbf{PASS (Ours)} & \textbf{$\Delta$ rel.} \\
\midrule
AIME24 & $1106 \pm 4$ & \cellcolor{mylightblue}$1118 \pm 14$ & $+1.1\%$ \\
AIME25 & $1162 \pm 5$ & \cellcolor{mylightblue}$1185 \pm 21$ & $+2.0\%$ \\
AMC23 & $935 \pm 9$ & \cellcolor{mylightblue}$946 \pm 10$ & $+1.1\%$ \\
GSM8K & $552 \pm 3$ & \cellcolor{mylightblue}$\mathbf{486 \pm 16}$ & $\mathbf{-11.9\%}$ \\
MATH & $690 \pm 3$ & \cellcolor{mylightblue}$688 \pm 17$ & $-0.2\%$ \\
Minerva & $880 \pm 3$ & \cellcolor{mylightblue}$\mathbf{754 \pm 33}$ & $\mathbf{-14.4\%}$ \\
Olympiad & $872 \pm 1$ & \cellcolor{mylightblue}$896 \pm 16$ & $+2.8\%$ \\
\midrule
\textbf{Average} & $885 \pm 1$ & \cellcolor{mylightblue}$\mathbf{868 \pm 18}$ & $\mathbf{-2.0\%}$ \\
\bottomrule
\end{tabular}
\end{table}

\begin{revision}
\paragraph{Comparing process-supervised estimators.} To test how the same PRM feedback is used, we compare six estimators with the same actor, scorer, data, and training budget (Table~\ref{tab:revision_summary}). This matched cohort uses three independent training seeds and 32 evaluation samples per problem. We compute benchmark-equal averages within seed, then their mean and sample SD; lengths retokenize saved responses. It is reported separately from Tables~\ref{tab:math_main}--\ref{tab:math_length}, whose runs and evaluation aggregation differ. Appendix~\ref{app:revision_results} specifies the protocols and complete per-benchmark results.

PASS achieves $36.39\pm0.54$ macro pass@1, exceeding Step-GRPO~\citep{li2026stepgrpo} by 6.39 pp. Step-GRPO aggregates PRM feedback into a trajectory reward, providing a different route from the same process scores to GRPO advantages; implementation settings appear in Appendix~\ref{app:revision_results}. PASS generates longer responses (863.7 versus 435.9 tokens), so this accuracy comparison also involves different realized inference costs.

\begin{table}[htbp]
\ifrevisionhighlight\color{blue}\fi
\centering
\caption{Matched mathematics comparisons. Accuracy is macro pass@1 (\%), mean $\pm$ sample SD over three independent training seeds. Length is the benchmark-equal mean retokenized response length.}
\label{tab:revision_summary}
\begin{tabular}{@{}lrr@{}}
\toprule
\textbf{Method} & \textbf{pass@1 (\%)} & \textbf{Length (tokens)} \\
\midrule
AF direct & $18.73\pm6.84$ & 419.0 \\
Native PRM-step RTG & $20.89\pm0.76$ & 770.1 \\
Native PRM-step RTG+DL & $33.44\pm0.66$ & 886.3 \\
CV+RTG & $33.22\pm0.41$ & 968.8 \\
\rowcolor{mylightblue}
\textbf{PASS (Ours)} & $36.39\pm0.54$ & 863.7 \\
Step-GRPO & $30.00\pm2.66$ & 435.9 \\
\bottomrule
\end{tabular}
\end{table}

\paragraph{Credit units and residual normalization.} With RTG+DL fixed, CV improves over native PRM-step accumulation by 2.95 pp while reducing mean length by $2.5\%$. Thus the PRM's scoring steps and PASS's shared credit units yield different learning outcomes even with the same scorer and residual-normalization rule. With CV+RTG fixed, adding DL raises accuracy by 3.17 pp and reduces length by $10.8\%$. Without DL, both native-step and CV accumulation perform less well, and their gap is larger. These conditional effects support considering credit units and residual scaling jointly: changing endpoints changes both the sampled values and the horizon over which RTG operates.

\paragraph{Learning behavior.} Figure~\ref{fig:revision_dynamics} complements the final-checkpoint comparison with correctness reward, PRM reward, and response-length trajectories. PASS maintains higher correctness reward over much of training while the two PRM-reward curves follow similar trends, and its responses are shorter than CV+RTG's. This pattern is consistent with a benefit from how feedback is converted into advantages, beyond simply increasing the scorer's output. The curves describe learning behavior rather than establish a causal mechanism. Appendix~\ref{app:revision_dynamics} also quantifies the shared endpoints, chunk lengths, and uncovered tails used by the estimator.
\end{revision}

\paragraph{Exponent and scorer scale.} Single-seed configuration studies in Appendix~\ref{app:math_sensitivity} examine $k\in\{1.0,0.7,0.5,0.3\}$ and PRMs of 7B and 72B parameters. The exponent sweep spans $35.1$--$36.9$ average pass@1, and the scorer-scale comparison increases average pass@1 from $35.7$ to $37.3$. These studies complement the main comparison by varying the strength of residual normalization and the source of process feedback.

\subsection{Multi-hop Question Answering with Teacher-Derived Signals}
\label{sec:exp_hotpotqa}

\paragraph{Setup.} We train Qwen2.5-7B-Instruct~\citep{qwen2025qwen25} on a fixed 5k-prompt HotpotQA distractor subset and evaluate on HotpotQA~\citep{yang2018hotpotqa}, 2WikiMultihopQA~\citep{ho2020constructing}, and MuSiQue~\citep{trivedi2022musique}. Qwen2.5-32B-Instruct supplies the teacher signals in Equations~\ref{eq:opd_signal} and \ref{eq:prcvdl_gopd_signal}; G-OPD uses the initial student as base and $\lambda=1.25$. Training runs for one epoch with rollout batches of 128 prompts, 16 responses per prompt, training batches of 128 trajectories, and learning rate $10^{-6}$. Each configuration uses three training seeds. Per-dataset entries report means and population standard deviations weighted by available evaluation-sample counts; averages give the three datasets equal weight. Evaluation targets 16 responses per question and uses SQuAD-normalized answer matching. Coverage and full protocol details appear in Appendix~\ref{app:hparams}.

\paragraph{Advantage-construction controls.} We compare \textbf{\ding{182} outcome-only}, using AF with $w_p=0$; \textbf{\ding{183} process-only}, using the normalized teacher-derived trace without outcome or format rewards; \textbf{\ding{184} AF-only}, combining all three channels through Equation~\ref{eq:prcvdl_advantage_shaping}; and \textbf{\ding{185} PASS}, applying CV, RTG, and DL to the fused trace. The first three use direct token advantages without RTG. \rev{We also compare with GR$^3$-style~\citep{li2026tackling}; its implementation and hyperparameters are specified alongside the other configurations in Appendix~\ref{app:hparams}.}

\paragraph{Overall performance.} Under Masked-Norm, PASS reaches $57.87$ average pass@1 with OPD and $57.63$ with G-OPD (Table~\ref{tab:hotpotqa_main}). It improves over AF-only by $3.40$ and $1.93$ pp and over outcome-only by $0.87$ and $0.63$ pp, respectively. Process-only training scores $54.33$ and $54.57$, showing that the teacher-derived signal alone is less effective than outcome-only training in this setting. The GR$^3$-style control scores $54.50$, close to AF-only's $54.47$. Full PASS achieves the highest mean on each benchmark for both signals.

\begin{table}[htbp]
\centering
\caption{pass@1 (\%) on multi-hop QA under Masked-Norm. Deltas compare with outcome-only (\ding{182}) and AF-only (\ding{184}). Implementation settings appear in Appendix~\ref{app:hparams}.}
\label{tab:hotpotqa_main}
\resizebox{\textwidth}{!}{%
\begin{tabular}{@{}llcccccc@{}}
\toprule
\textbf{Rung} & \textbf{Method} & \textbf{MuSiQue} & \textbf{2Wiki} & \textbf{HotpotQA} & \textbf{Average} & \textbf{$\Delta_{\text{\ding{182}}}$} & \textbf{$\Delta_{\text{\ding{184}}}$} \\
\midrule
\ding{182} & GRPO (outcome-only) & $44.2 \pm 0.3$ & $64.9 \pm 0.2$ & $61.9 \pm 0.5$ & $57.00$ & -- & -- \\
\midrule
\multicolumn{8}{@{}l}{\emph{OPD process signal}}\\
\ding{183} & \quad Process-only OPD & $38.3 \pm 0.8$ & $63.8 \pm 0.4$ & $60.9 \pm 0.1$ & $54.33$ & $-2.67$ & $-0.14$ \\
\ding{184} & \quad AF-only OPD & $38.4 \pm 0.5$ & $63.9 \pm 0.2$ & $61.1 \pm 0.1$ & $54.47$ & $-2.53$ & -- \\
comp. & \quad GR$^3$-style + OPD ($\alpha=0.33$) & $38.1 \pm 1.1$ & $64.0 \pm 0.5$ & $61.4 \pm 0.4$ & $54.50$ & $-2.50$ & $+0.03$ \\
\rowcolor{mylightblue}
\ding{185} & \quad \textbf{PASS + OPD} & $\mathbf{45.3 \pm 0.7}$ & $\mathbf{65.6 \pm 0.2}$ & $\mathbf{62.7 \pm 0.2}$ & $\mathbf{57.87}$ & $\mathbf{+0.87}$ & $\mathbf{+3.40}$ \\
\midrule
\multicolumn{8}{@{}l}{\emph{G-OPD process signal}}\\
\ding{183} & \quad Process-only G-OPD & $39.2 \pm 1.4$ & $63.3 \pm 0.7$ & $61.2 \pm 0.3$ & $54.57$ & $-2.43$ & $-1.13$ \\
\ding{184} & \quad AF-only G-OPD & $41.3 \pm 1.2$ & $64.0 \pm 0.4$ & $61.8 \pm 0.3$ & $55.70$ & $-1.30$ & -- \\
\rowcolor{mylightblue}
\ding{185} & \quad \textbf{PASS + G-OPD} & $\mathbf{44.6 \pm 0.3}$ & $\mathbf{65.7 \pm 0.2}$ & $\mathbf{62.6 \pm 0.1}$ & $\mathbf{57.63}$ & $\mathbf{+0.63}$ & $\mathbf{+1.93}$ \\
\bottomrule
\end{tabular}%
}
\end{table}

\paragraph{Effect of normalization.} Table~\ref{tab:normalization_summary} compares both operators. With Abs-Max, PASS remains above process-only and AF-only in mean accuracy, scoring $56.93$ for OPD and $56.17$ for G-OPD. These values are $0.07$ and $0.83$ pp below outcome-only training. Normalization has recipe-dependent effects: Abs-Max increases AF-only OPD from $54.47$ to $55.90$, leaves AF-only G-OPD at $55.70$, and decreases the PASS means. Thus PASS improves on the process-signal controls under both operators, while the gain over outcome-only depends on normalization. Per-benchmark Abs-Max results are in Appendix~\ref{app:normalization}.

\begin{table}[htbp]
\centering
\caption{QA average pass@1 across normalization choices, with PASS response length in tokens. Outcome-only scores $57.00$ at 186.8 tokens under the shared control.}
\label{tab:normalization_summary}
\begin{tabular}{@{}llccccc@{}}
\toprule
\textbf{Signal} & \textbf{Normalizer} & \textbf{Process-only} & \textbf{AF-only} & \cellcolor{mylightblue}\textbf{PASS} & \textbf{$\Delta$ vs. outcome} & \textbf{Length} \\
\midrule
OPD & Masked-Norm & 54.33 & 54.47 & \cellcolor{mylightblue}57.87 & $+0.87$ & 239.0 \\
G-OPD & Masked-Norm & 54.57 & 55.70 & \cellcolor{mylightblue}57.63 & $+0.63$ & 249.7 \\
OPD & Abs-Max & 53.87 & 55.90 & \cellcolor{mylightblue}56.93 & $-0.07$ & 239.8 \\
G-OPD & Abs-Max & 54.43 & 55.70 & \cellcolor{mylightblue}56.17 & $-0.83$ & 221.0 \\
\bottomrule
\end{tabular}
\end{table}

\paragraph{Accuracy and response length.} The QA gains accompany longer outputs. Under Masked-Norm, PASS averages 239.0 tokens with OPD and 249.7 with G-OPD, compared with 186.8 for outcome-only. AF-only and GR$^3$-style OPD both average 163.2 tokens. In contrast to mathematics, QA therefore presents an accuracy--length trade-off; these comparisons do not match inference compute. Appendix~\ref{app:qa_dynamics} and Appendix~\ref{app:length} provide the complete length results and further training-time dynamics of rollout length, actor entropy, and teacher KL for the QA runs.

\begin{revision}
\paragraph{Accumulating token-derived feedback.} Five matched OPD configurations isolate RTG, credit granularity, and DL under Masked-Norm. Each uses three independent training seeds, a shared one-epoch training protocol, and 16 evaluation samples per question. As in the matched mathematics study, benchmark-equal averages are computed within seed before reporting the mean and sample SD. This cohort has its own PASS reference and aggregation, separate from Tables~\ref{tab:hotpotqa_main}--\ref{tab:normalization_summary}; full results and retokenized lengths appear in Appendix~\ref{app:revision_results}.

With CV fixed, adding RTG increases macro pass@1 from 54.60 to 57.57 (a 2.96-point gain before rounding), and adding DL brings it to 58.22 while reducing mean response length from 329.5 to 281.0 tokens. The same residual-scaling choice therefore improves the accuracy--length balance with dense teacher feedback as with PRM feedback. At token granularity, adding DL increases accuracy by 1.12 pp, reaching 58.41. This slightly exceeds CV PASS's 58.22 pass@1 mean.

At pass@8 and pass@16, CV PASS instead achieves higher means on all three QA benchmarks, with macro scores of 72.07 versus 71.04 and 74.99 versus 73.70, respectively (Tables~\ref{tab:revision_qa_pass8}--\ref{tab:revision_qa_pass16}). It also produces shorter responses on all three benchmarks, averaging 281.0 versus 335.0 tokens ($16.1\%$ fewer), with lower observed seed SD in macro pass@1 (0.22 versus 0.56). In contrast to mathematics, where CV improves mean pass@1 over native PRM-step RTG+DL, QA exposes a metric-dependent trade-off: token resolution slightly favors pass@1, while shared accumulation improves pass@8/16 with shorter outputs. The comparison relates credit resolution to the source signal rather than assuming that one granularity is best across tasks.
\end{revision}

\section{Discussion and Limitations}
\label{app:discussion}

\begin{revision}
\paragraph{From score geometry to policy advantages.} A fixed PRM or teacher does not define a unique learning signal. Normalization changes the reference point and relative scale of local feedback; fusion determines its interaction with outcome and format rewards; accumulation transforms those values into position-dependent credit. The QA normalization controls illustrate this dependence: PASS improves on process-only and AF-only under both operators, but its gain over outcome-only occurs under Masked-Norm rather than Abs-Max. Mean-centering and absolute-maximum scaling preserve different offsets in the process channel, which then enter the fused trace. The matched mathematics comparisons provide complementary evidence: with the scorer held fixed, shared CV accumulation and residual normalization improve over native-step estimators. The upstream evaluator and the mapping from its scores to advantages are therefore distinct, experimentally consequential parts of process-supervised RL.

\paragraph{A shared computational coordinate for credit.} CV should be viewed as defining a practical credit coordinate system from the supervision signal, rather than recovering latent reasoning steps. A batch-shared partition supplies commonly defined positional units across rollouts; changes in the fused feedback set their resolution without an auxiliary segmenter or a fixed token-width rule. This separates the granularity at which a signal is scored from the granularity at which its future contributions are accumulated. The construction applies to both piecewise-constant PRM feedback and token-varying teacher feedback, while allowing their numerical variation to induce different partitions. A supplementary single-seed diagnostic finds that shared CV endpoints retain over $80\%$ of native PRM boundaries (Appendix~\ref{app:revision_dynamics}), suggesting that the shared partition can preserve much of the native scoring structure while introducing additional accumulation units.

The shared trace also introduces an approximation. Individual score changes can cancel in its masked mean, and changes in the active-response set can create boundaries even when individual traces remain constant. Consequently, broadcasting endpoint values need not preserve every response's local score variation. Tokens beyond the last valid shared endpoint receive zero direct advantage; Appendix~\ref{app:revision_dynamics} quantifies this coverage. The controlled comparisons test the utility of this concrete estimator, including its endpoint sampling, rather than semantic alignment. Same-density fixed-length comparisons and batch-resampling studies would further clarify which aspects of the shared, adaptive partition matter.

\paragraph{Coupled operators, signal-dependent trade-offs.} AF determines the score geometry, CV converts its variation into accumulation units, and DL sets the scale of credit over the remaining units. Changing channel scale can therefore change the selected boundaries as well as advantage magnitudes, while changing the partition changes the residual horizon used by DL. The mathematics ablations reflect this coupling: the native-step/CV difference depends on whether DL is present. With DL fixed, CV improves both accuracy and length relative to native PRM-step accumulation. With teacher-derived token feedback, token RTG+DL instead has a slightly higher pass@1 mean, while CV PASS has higher pass@8/16, shorter outputs, and lower observed pass@1 seed SD. These results support shared, reward-adaptive credit units as a useful design choice whose benefits depend on the supervision source, rather than a universally optimal resolution.

DL acts on cumulative credit, so its effect on generated length emerges through learning. Its divisor counts remaining credit units rather than imposing a direct token penalty. At $k=1$, it averages future chunk scores; at $k<1$, it retains more horizon dependence. This helps interpret the different comparisons: PASS is shorter than outcome-only on mathematics but longer on QA, while DL reduces length relative to CV+RTG in both matched studies. The mathematics training curves further show higher correctness with similar PRM-reward trends and shorter responses. Together these observations connect advantage construction to learning behavior, while leaving causal attribution and inference-compute-matched accuracy comparisons open.

\paragraph{Scope and generalizability.} Our experiments cover mathematical reasoning with learned PRM feedback and multi-hop QA with teacher-derived token feedback, using Qwen2.5-based 7B policies. These settings demonstrate the relevance of advantage construction across distinct supervision sources. Whether the observed benefits extend to other model families, feedback mechanisms, and multi-turn agentic settings remains an open question.
\end{revision}

\section{Conclusion}
\label{sec:conclusion}

\begin{revision}
PASS organizes existing process supervision into GRPO advantages through separate channel fusion, batch-shared reward-adaptive credit units, and residual-count normalization. The outcome-only comparisons establish its overall effectiveness, while matched process-supervision comparisons and component ablations show that the choice of accumulation units and horizon scaling matters even with the feedback source fixed. Mathematics and teacher-derived QA exhibit different accuracy--length behavior, making credit resolution a signal-dependent design choice. These findings motivate treating the conversion from local evaluations to policy advantages as a central part of reasoning RL, alongside the quality of the evaluator itself.
\end{revision}

\newpage
\subsection*{Ethics statement}

This work uses existing mathematical-reasoning and question-answering datasets and open-weight models. It does not involve new human-subject data collection. Advantage shaping is not a safety or factuality guarantee: emphasizing a process signal may also emphasize errors or biases in its source. Our empirical scope is limited to the evaluated benchmarks and training recipes.

\subsection*{Reproducibility statement}

Section~\ref{sec:method} defines the PASS operators and the two teacher-derived signals. Appendix~\ref{app:algorithms} specifies the full computation, including masks, response endpoints, and token alignment. Training configurations, evaluation sampling counts, and statistical aggregation are provided in Appendix~\ref{app:hparams}. The single-seed exponent and PRM-scale studies are described in Appendix~\ref{app:math_sensitivity}, and Appendix~\ref{app:normalization} gives per-benchmark results for the alternative normalizer.

\subsection*{AI use statement}

Generative AI tools assisted with code analysis, result-aggregation scripts, manuscript drafting and editing, and checking and revising mathematical arguments and method descriptions. Experimental summaries were computed from evaluation outputs. The authors are responsible for the final text, technical claims, and experimental results.

\newpage

\bibliographystyle{iclr2027_conference}
\bibliography{refs}

\newpage

\appendix

\section{Fixed-Score Analysis of Residual Normalization}
\label{app:proofs}

\rev{We suppress the response index and use chunk scores $c_j$ as in Section~\ref{subsec:method_dl}.} Fix $K\geq1$ selected chunks with scores $\rev{c_1,\ldots,c_K}$, total $R=\rev{\sum_j c_j}$, and exponent $0<k\leq1$. Holding the existing chunk scores, segmentation, and normalization statistics fixed, append one new chunk of score $\delta$. For the path-score diagnostic $S=R/K^k$,
\begin{equation}
\Delta S=\frac{R+\delta}{(K+1)^k}-\frac{R}{K^k}
=\frac{\delta-R[(1+1/K)^k-1]}{(K+1)^k}.
\end{equation}
The denominator is positive, so Equation~\ref{eq:prcvdl_dl_marginal} follows for signed $R$ and $\delta$. At $k=1$, $\Delta S>0$ exactly when $\delta>R/K$. This is an average-score comparison, not a hard acceptance test performed by the algorithm. For $k<1$ the exact threshold differs from the current average. If $R=O(K)$ and $\delta=O(1)$, expansion gives
\begin{equation}
\Delta S=K^{-k}\left(\delta-k\frac{R}{K}\right)+O(K^{-k-1}).
\end{equation}

\begin{revision}
\paragraph{Repeated contributions.} With unchanged prefix scores, adding $m$ chunks each scoring at least $\xi>0$ increases the unnormalized cumulative score by at least $m\xi$. Repeated constant contributions under residual-count attenuation give
\begin{equation}
S_m=\frac{R+m\xi}{(K+m)^k}\sim\xi m^{1-k}\longrightarrow+\infty
\quad(0<k<1).
\end{equation}
At $k=1$, this particular score sequence is decreasing if $\xi<R/K$, increasing if $\xi>R/K$, and constant at equality. There is no general finite-length optimum implied by $k<1$. Moreover, appending a semantic reasoning step need not add a detected chunk: equal-valued adjacent segments can merge.

\end{revision}

\paragraph{Scope of the analysis.} The diagnostic equals the first-endpoint value of the residual estimator. Other positions use their own remaining counts, and policy optimization additionally includes clipping, reference regularization, and masked loss reduction. Moreover, resampling responses changes normalization statistics and shared boundaries. The fixed-score relations therefore characterize the role of $k$ within advantage construction, rather than establish convergence or determine the direction of trained response-length changes.

\section{Algorithmic Specification of PASS}
\label{app:algorithms}

The algorithms below specify channel fusion, shared endpoint selection, and residual-count normalization. Rewards, log probabilities, and the resulting advantages are held fixed during the policy update. A same-prompt group contains $G$ responses; $\mathcal B$ is the batch used for process normalization and shared segmentation. Missing sequence positions contribute neither values nor counts.

\begin{algorithm}[h]
\caption{Channel handling for QA and mathematical full PASS.}
\label{alg:af}
\begin{algorithmic}[1]
\Require Process traces $s_i$, response masks $M_i$, scalar rewards $r_{o,i},r_{f,i}$, group size $G$, weights $w_p,w_o,w_f$, and scale $\gamma$.
\Ensure Masked fused traces $a_i$.
\State Normalize process values over the valid token pool in $\mathcal B$ by Equation~\ref{eq:pre_normalized_step_signal}.
\State Normalize $r_o$ and $r_f$ separately within each same-prompt group of $G$, obtaining $z_o,z_f$.
\State $W\gets w_p+w_o+w_f$
\State Initialize $a$ to zero.
\For{each response $i$}
  \If{$r_{f,i}>0$}
    \State $a_{i,t}\gets\gamma(w_pz_{p,i,t}+w_oz_{o,i}+w_fz_{f,i})$ on valid response positions.
  \Else
    \State $a_{i,t}\gets\gamma Wz_{f,i}$ on valid response positions.
  \EndIf
\EndFor
\State \Return $a$
\end{algorithmic}
\end{algorithm}

\begin{algorithm}[h]
\caption{Shared endpoint detection for QA and mathematical full PASS.}
\label{alg:cv}
\begin{algorithmic}[1]
\Require Fused traces $a_i$, masks $M_i$; $\eta=5\times10^{-3}$ for QA or $10^{-8}$ for math.
\Ensure Retained endpoint sets $B_i$ and scores $\rev{c_{i,j}}$.
\State Find the union response interval $[t_0,t_1]$ across $\mathcal B$.
\State $u_t\gets\sum_i M_{i,t}a_{i,t}/\max(1,\sum_iM_{i,t})$ for each $t$ in the interval.
\State $B\gets\varnothing$
\For{$t=t_0+1,\ldots,t_1$}
  \If{$|u_t-u_{t-1}|>\eta$}
    \State Add $t-1$ to $B$ \Comment{adjacent-value comparison}
  \EndIf
\EndFor
\State Add $t_1$ to $B$.
\For{each response $i$}
  \State $B_i\gets\{b\in B:M_{i,b}=1\}$ in increasing order.
  \State $K_i\gets|B_i|$; $\rev{c_{i,j}}\gets a_{i,b_{i,j}}$ for $j=1,\ldots,K_i$.
\EndFor
\State \Return $B_i,\rev{c_{i,j}}$ \Comment{do not append a per-response final endpoint}
\end{algorithmic}
\end{algorithm}

\paragraph{Masks and endpoints.} Coordinates in Algorithm~\ref{alg:cv} are positions in the padded sequences, so the positional mean uses the number of valid responses at each position rather than a constant denominator. An all-zero mask contributes no values or counts and retains no endpoints; its stored sequence span participates in constructing the shared interval. Each response uses the valid subset of the shared endpoints, without adding a response-specific terminal endpoint. Consequently the broadcast leaves a zero-filled tail after the last retained endpoint. If no endpoint is retained, the returned advantage is zero throughout. Mathematical process scores use a PRM-aligned response mask that excludes the terminal \rev{end-of-sequence (EOS) token}, and the final advantages are shifted into next-token prediction coordinates; outcome and format normalization use their sequence-level response masks.

\begin{algorithm}[h]
\caption{Chunk RTG, residual normalization, and token broadcast.}
\label{alg:dl}
\begin{algorithmic}[1]
\Require Endpoint scores $\rev{c_{i,j}}$, ordered $B_i$, union interval $[t_0,t_1]$, exponent $k$.
\Ensure Token advantages $A_i$.
\For{each response $i$}
  \State Initialize $A_i$ to zero; $R\gets0$.
  \For{$j=K_i,K_i-1,\ldots,1$}
    \State $R\gets R+\rev{c_{i,j}}$
    \State $v_{i,j}\gets R/(K_i-j+1)^k$
  \EndFor
  \State $b_{i,0}\gets t_0-1$
  \For{$j=1,\ldots,K_i$}
    \State Fill $A_{i,t}\gets v_{i,j}$ for $b_{i,j-1}<t\leq b_{i,j}$.
  \EndFor
\EndFor
\State \Return $A$ \Comment{unfilled tails stay zero; the loss applies its response mask}
\end{algorithmic}
\end{algorithm}

\begin{algorithm}[h]
\caption{End-to-end PASS advantage construction.}
\label{alg:pass_pipeline}
\begin{algorithmic}[1]
\Require Rollout batch, channel values, masks, weights, scale $\gamma$, tolerance $\eta$, and exponent $k$.
\State $a\gets\text{ChannelHandling}(\cdot)$ using Algorithm~\ref{alg:af}.
\State $(B_i,\rev{c_{i,j}})\gets\text{SharedEndpoints}(a,M,\eta)$ using Algorithm~\ref{alg:cv}.
\State $A\gets\text{ResidualRTG}(B_i,\rev{c_{i,j}},k)$ using Algorithm~\ref{alg:dl}.
\If{mathematical implementation}
  \State Shift the PRM-coordinate advantages one position left into next-token loss coordinates.
\EndIf
\State Store $A$ as a fixed rollout target for the clipped surrogate.
\end{algorithmic}
\end{algorithm}

\rev{Diagnostic-only terminal completion for depth-wise chunk lengths does not modify Algorithm~\ref{alg:cv} or the actual RTG/broadcast. Tail coverage is measured before that completion, using the PRM mask excluding terminal EOS in mathematics; all zero-advantage tokens are counted separately.}

The QA direct controls return either the process trace, the fused trace, or the outcome/format-only fused trace without Algorithms~\ref{alg:cv} and \ref{alg:dl}. The GR$^3$-style control applies its response-level length factor to the fused trace. Full PASS uses all stages, with $\gamma=1/W$, $\eta=0.005$, $k=0.7$ in QA and $\gamma=1$, $\eta=10^{-8}$, $k=1$ in the main mathematics comparison.

\FloatBarrier
\section{Additional Results on Mathematical Reasoning}
\label{app:math_sensitivity}

We report pass@8 and problem-level uncertainty for the main mathematical comparisons, followed by single-seed studies of the normalization exponent and process reward model scale.

\paragraph{pass@8.} PASS achieves $55.65\pm0.64$ versus $52.89\pm0.52$ for ORM, a mean improvement of $2.76$ pp (Table~\ref{tab:math_pass8}). The largest gains are on AIME24 ($+6.01$) and AIME25 ($+8.39$), where successful solutions remain less frequent. Improvements are smaller on AMC23, GSM8K, and Minerva ($+0.24$, $+0.23$, and $+0.51$), with overlapping observed run ranges. Thus the increase in multi-sample success is concentrated more strongly on the AIME tasks than the pass@1 gains.

\begin{table}[h]
\centering
\caption{pass@8 (\%) on the same mathematical runs as Table~\ref{tab:math_main}. Deltas are computed before rounding. The average difference is $+2.76$; observed run ranges overlap on AMC23, GSM8K, and Minerva.}
\label{tab:math_pass8}
\begin{tabular}{@{}lrrr@{}}
\toprule
\textbf{Benchmark} & \textbf{GRPO(ORM)} & \cellcolor{mylightblue}\textbf{PASS (Ours)} & \textbf{$\Delta$} \\
\midrule
AIME24 & $38.29 \pm 1.29$ & \cellcolor{mylightblue}$\mathbf{44.30 \pm 1.88}$ & $+6.01$ \\
AIME25 & $18.17 \pm 0.39$ & \cellcolor{mylightblue}$\mathbf{26.56 \pm 1.36}$ & $+8.39$ \\
AMC23 & $82.18 \pm 2.05$ & \cellcolor{mylightblue}$82.42 \pm 1.40$ & $+0.24$ \\
GSM8K & $95.44 \pm 0.18$ & \cellcolor{mylightblue}$95.68 \pm 0.17$ & $+0.23$ \\
MATH & $69.78 \pm 0.12$ & \cellcolor{mylightblue}$\mathbf{71.63 \pm 0.09}$ & $+1.84$ \\
Minerva & $47.32 \pm 0.52$ & \cellcolor{mylightblue}$47.83 \pm 0.66$ & $+0.51$ \\
Olympiad & $19.04 \pm 0.10$ & \cellcolor{mylightblue}$\mathbf{21.12 \pm 0.17}$ & $+2.07$ \\
\midrule
\textbf{Average} & $52.89 \pm 0.52$ & \cellcolor{mylightblue}$\mathbf{55.65 \pm 0.64}$ & $\mathbf{+2.76}$ \\
\bottomrule
\end{tabular}
\end{table}

\paragraph{Problem-level uncertainty.} AIME24/25 contain 30 problems each and AMC23 40. Table~\ref{tab:math_bootstrap} reports 95\% bootstrap intervals using 2000 resamples of whole problems, carrying each problem's full response set. These intervals describe uncertainty across problems conditional on each trained checkpoint. They complement the cross-training-seed standard deviations in the main results; they are marginal score intervals rather than intervals for the difference between methods.

\begin{table}[h]
\centering
\caption{Problem-level bootstrap 95\% intervals for pass@1 (\%), conditional on each checkpoint, using 2000 resamples.}
\label{tab:math_bootstrap}
\resizebox{0.8\textwidth}{!}{%
\begin{tabular}{@{}lccc@{}}
\toprule
\textbf{Configuration} & \textbf{AIME24} & \textbf{AIME25} & \textbf{AMC23} \\
\midrule
GRPO(ORM), seed 1111 & $12.43$ \small{[6.72, 19.21]} & $4.74$ \small{[1.41, 9.43]}  & $49.17$ \small{[40.08, 58.65]} \\
GRPO(ORM), seed 2222 & $12.29$ \small{[5.99, 19.32]} & $4.22$ \small{[1.46, 8.07]}  & $46.60$ \small{[36.52, 56.02]} \\
\midrule
\rowcolor{mylightblue}\cellcolor{white}
\textbf{PASS}, seed 1 & $\mathbf{18.05}$ \small{[10.29, 27.08]} & $\mathbf{8.78}$ \small{[3.75, 15.23]} & $\mathbf{53.93}$ \small{[43.09, 64.38]} \\
\rowcolor{mylightblue}\cellcolor{white}
\textbf{PASS}, seed 2 & $\mathbf{16.35}$ \small{[9.27, 24.14]} & $\mathbf{9.51}$ \small{[4.14, 16.20]} & $\mathbf{54.96}$ \small{[43.96, 65.80]} \\
\rowcolor{mylightblue}\cellcolor{white}
\textbf{PASS}, seed 42 & $\mathbf{16.46}$ \small{[9.27, 25.10]} & $\mathbf{9.27}$ \small{[3.75, 16.25]} & $\mathbf{52.27}$ \small{[41.64, 63.36]} \\
\bottomrule
\end{tabular}%
}
\end{table}

\paragraph{Sensitivity to the normalization exponent.} We conduct the exponent and PRM-scale studies with one training seed per configuration and compare results within each table. Table~\ref{tab:math_ablation_k} varies $k\in\{1.0,0.7,0.5,0.3\}$ for PRM-supervised mathematical reasoning. Average pass@1 ranges from $35.1$ to $36.9$, a span of $1.8$ pp, with the highest observed value at $k=0.7$. The main mathematical comparison uses $k=1$, corresponding to residual averaging; smaller values retain more dependence on the remaining horizon. The single-seed sweep illustrates the effect of this choice without establishing a seed-robust ranking.

\begin{table}[h]
\centering
\caption{Sensitivity to the residual-normalization exponent $k$, with one training seed per configuration. Metrics are pass@1/pass@8 (\%).}
\label{tab:math_ablation_k}
\resizebox{\textwidth}{!}{%
\begin{tabular}{@{}lcccccccc@{}}
\toprule
\textbf{Decay Factor} & \textbf{AIME24} & \textbf{AIME25} & \textbf{AMC23} & \textbf{GSM8K} & \textbf{MATH} & \textbf{Minerva} & \textbf{Olympiad} & \textbf{Average} \\
\midrule
$k=1.0$   & 14.8/39.1 & \textbf{9.8}/27.4 & 53.8/\textbf{83.7} & 77.5/\textbf{95.4} & 55.1/71.3 & 27.3/47.2 & 11.5/21.1 & 35.7/55.0 \\
\rowcolor{mylightblue}
$k=0.7$      & \textbf{17.8/47.1} & 9.3/24.7 & \textbf{56.2}/81.8 & \textbf{78.3}/95.2 & \textbf{55.7/71.4} & \textbf{29.1/48.6} & \textbf{11.8/21.2} & \textbf{36.9}/55.7 \\
$k=0.5$ (Sqrt-L)      & 15.9/44.2 & 9.9/\textbf{27.9} & 53.4/82.6 & 74.7/95.2 & 52.8/71.0 & 27.6/48.2 & 11.1/21.2 & 35.1/55.8 \\
$k=0.3$               & 15.7/45.9 & 9.4/26.7 & 54.1/83.0 & 77.1/95.4 & 54.8/71.2 & 28.5/48.1 & 11.5/20.9 & 35.9/\textbf{55.9} \\
\bottomrule
\end{tabular}%
}
\end{table}

\paragraph{Effect of process reward model scale.} Table~\ref{tab:math_ablation_prm} compares 7B and 72B PRMs~\citep{zhang2025lessons} at $k=1$, keeping the actor at 7B scale. The 72B scorer increases average pass@1 from $35.7$ to $37.3$ and pass@8 from $55.0$ to $55.8$. This result indicates that the quality of the upstream scorer remains consequential when its outputs are processed by PASS.

\begin{table}[h]
\centering
\caption{Effect of PRM scale at $k=1$, with one training seed per configuration. Metrics are pass@1/pass@8 (\%).}
\label{tab:math_ablation_prm}
\resizebox{\textwidth}{!}{%
\begin{tabular}{@{}lcccccccc@{}}
\toprule
\textbf{PRM Scale} & \textbf{AIME24} & \textbf{AIME25} & \textbf{AMC23} & \textbf{GSM8K} & \textbf{MATH} & \textbf{Minerva} & \textbf{Olympiad} & \textbf{Average} \\
\midrule
7B PRM                 & 14.8/39.1 & \textbf{9.8}/27.4 & 53.8/\textbf{83.7} & 77.5/95.4 & 55.1/71.3 & 27.3/47.2 & 11.5/21.1 & 35.7/55.0 \\
\rowcolor{mylightblue}
\textbf{72B PRM}       & \textbf{17.2/43.7} & 9.1/\textbf{27.4} & \textbf{57.2}/83.5 & \textbf{79.8/95.5} & \textbf{56.9/71.6} & \textbf{29.2/47.5} & \textbf{12.0/21.5} & \textbf{37.3/55.8} \\
\bottomrule
\end{tabular}%
}
\end{table}

\FloatBarrier
\section{Group Standardization and Process Signals}
\label{app:normalization}
\label{sec:exp_hotpotqa_scaled}

Masked-Norm centers the process trace at its masked mean, whereas Abs-Max preserves the zero and sign of each sampled value while rescaling its magnitude. The resulting traces differ both in scale and in their reference point. Table~\ref{tab:hotpotqa_scaled} gives the per-benchmark Abs-Max results corresponding to the normalization summary in Table~\ref{tab:normalization_summary}.

\paragraph{Setup.} The process-only, AF-only, and PASS configurations are evaluated under both operators with the same actor, teacher, and training budget. The outcome-only runs are shared because their advantages do not depend on a process normalizer. The GR$^3$-style control uses Masked-Norm. The QA aggregation convention is defined in \S\ref{sec:exp_hotpotqa}.

\begin{table}[h]
\centering
\caption{pass@1 (\%) on multi-hop QA under Abs-Max Scaling. PASS has the highest mean among the tested process-signal configurations, but does not exceed outcome-only training. Deltas use the shared outcome-only and corresponding AF-only controls.}
\label{tab:hotpotqa_scaled}
\resizebox{\textwidth}{!}{%
\begin{tabular}{@{}llcccccc@{}}
\toprule
\textbf{Rung} & \textbf{Method} & \textbf{MuSiQue} & \textbf{2Wiki} & \textbf{HotpotQA} & \textbf{Average} & \textbf{$\Delta_{\text{\ding{182}}}$} & \textbf{$\Delta_{\text{\ding{184}}}$} \\
\midrule
\ding{182} & GRPO (outcome-only) & $44.2 \pm 0.3$ & $64.9 \pm 0.2$ & $61.9 \pm 0.5$ & $57.00$ & -- & -- \\
\midrule
\multicolumn{8}{@{}l}{\emph{OPD process signal}}\\
\ding{183} & \quad Process-only OPD & $38.4 \pm 0.4$ & $62.7 \pm 0.4$ & $60.5 \pm 0.2$ & $53.87$ & $-3.13$ & $-2.03$ \\
\ding{184} & \quad AF-only OPD & $42.1 \pm 0.5$ & $63.9 \pm 0.2$ & $61.7 \pm 0.4$ & $55.90$ & $-1.10$ & -- \\
\rowcolor{mylightblue}
\ding{185} & \quad \textbf{PASS + OPD} & $\mathbf{44.1 \pm 0.5}$ & $\mathbf{64.7 \pm 0.2}$ & $\mathbf{62.0 \pm 0.3}$ & $\mathbf{56.93}$ & $-0.07$ & $\mathbf{+1.03}$ \\
\midrule
\multicolumn{8}{@{}l}{\emph{G-OPD process signal}}\\
\ding{183} & \quad Process-only G-OPD & $39.5 \pm 0.7$ & $63.1 \pm 0.3$ & $60.7 \pm 0.1$ & $54.43$ & $-2.57$ & $-1.27$ \\
\ding{184} & \quad AF-only G-OPD & $41.7 \pm 0.7$ & $63.9 \pm 0.4$ & $61.5 \pm 0.5$ & $55.70$ & $-1.30$ & -- \\
\rowcolor{mylightblue}
\ding{185} & \quad \textbf{PASS + G-OPD} & $\mathbf{42.7 \pm 0.8}$ & $\mathbf{64.0 \pm 0.7}$ & $\mathbf{61.8 \pm 0.1}$ & $\mathbf{56.17}$ & $-0.83$ & $\mathbf{+0.47}$ \\
\bottomrule
\end{tabular}%
}
\end{table}

\paragraph{Normalization-dependent performance.} Moving from Masked-Norm to Abs-Max raises AF-only OPD from $54.47$ to $55.90$ ($+1.43$), while AF-only G-OPD stays at $55.70$ at the displayed precision. Process-only decreases from $54.33$ to $53.87$ for OPD and from $54.57$ to $54.43$ for G-OPD. The corresponding PASS means decrease from $57.87$ to $56.93$ and from $57.63$ to $56.17$. These changes show that the effect of normalization depends on the advantage recipe.

PASS has higher mean accuracy than the process-only and AF-only configurations under both standardizers. Its differences from outcome-only are $+0.87$ and $+0.63$ with Masked-Norm, and $-0.07$ and $-0.83$ with Abs-Max. Normalization therefore changes the comparison with outcome-only training while preserving PASS's mean advantage over the tested process-signal controls.

\FloatBarrier
\section{Training Dynamics on Multi-hop QA}
\label{app:qa_dynamics}

We complement the endpoint evaluation in Section~\ref{sec:exp_hotpotqa} with three training-time observables from the same Masked-Norm runs: rollout length, actor entropy, and token-level KL divergence to the teacher. These quantities do not identify a causal mechanism by themselves, but together they provide a view of how the learned policies evolve under different advantage-construction recipes.

\begin{figure}[htbp]
\centering
\includegraphics[width=\textwidth]{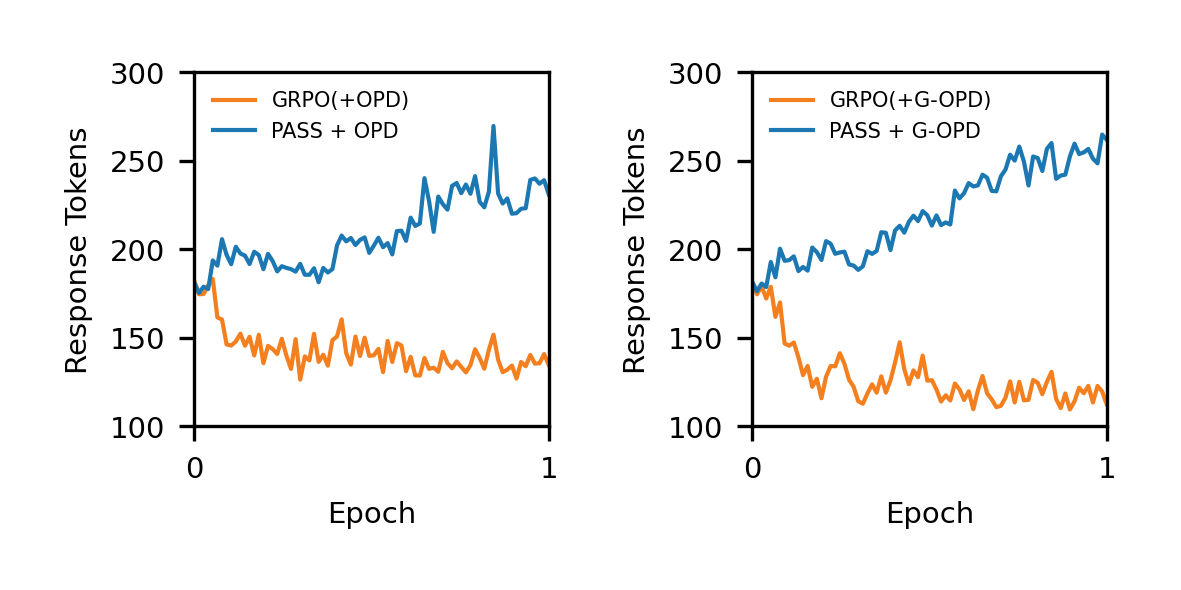}
\caption{Training-time mean rollout length on HotpotQA across one training epoch for the Masked-Norm OPD and G-OPD configurations. Left: OPD. Right: G-OPD.}
\label{fig:qa_training_length}
\end{figure}

Figure~\ref{fig:qa_training_length} shows the mean rollout length over training. Starting from the same actor, the process-signal controls and PASS evolve toward substantially different response-length regimes. For both OPD and G-OPD, PASS produces progressively longer rollouts, consistent with the longer evaluation-time responses reported in Section~\ref{sec:exp_hotpotqa}. This behavior is also consistent with our interpretation of residual normalization as credit scaling rather than a direct length penalty: changing the residual-count normalization can alter the learned response length in either direction depending on the resulting advantage structure.

\begin{figure}[htbp]
\centering
\includegraphics[width=\textwidth]{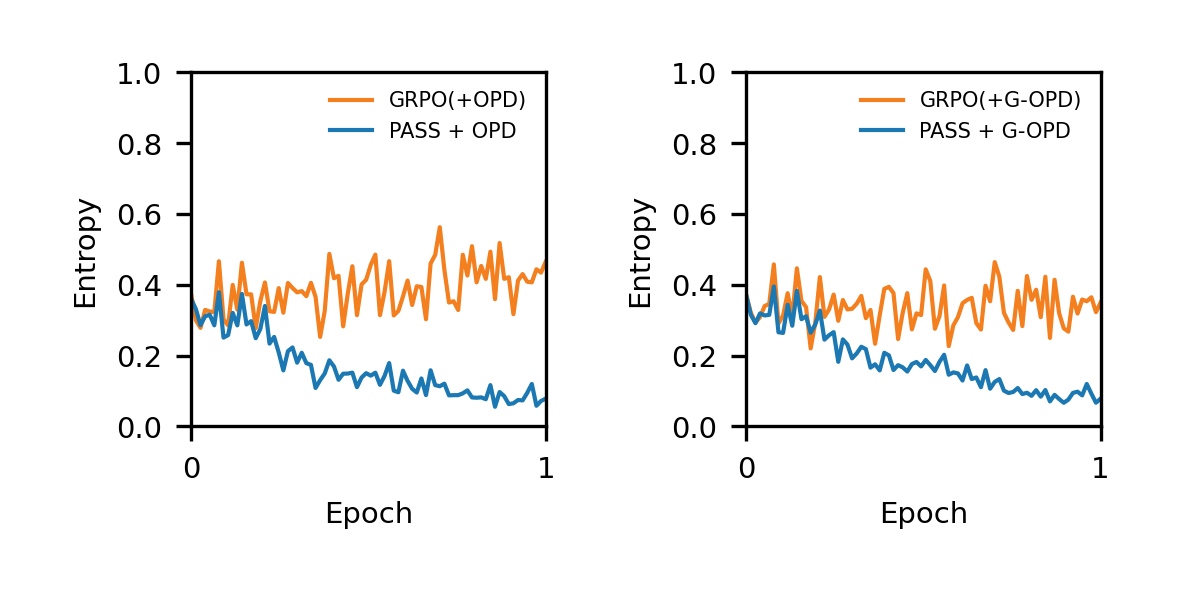}
\caption{Training-time token-level actor entropy on HotpotQA for the same runs as Figure~\ref{fig:qa_training_length}. Left: OPD. Right: G-OPD.}
\label{fig:qa_training_entropy}
\end{figure}

Figure~\ref{fig:qa_training_entropy} provides a complementary view through the actor's token-level entropy. The PASS runs exhibit a pronounced decrease in entropy over training, whereas the corresponding controls remain at substantially higher entropy. Read together with Figure~\ref{fig:qa_training_length}, the resulting policies are therefore not only longer but also more locally concentrated in their token distributions. This joint pattern argues against interpreting the additional tokens as merely high-entropy continuation, although it does not establish that every additional token contributes useful reasoning.

\begin{figure}[htbp]
\centering
\includegraphics[width=\textwidth]{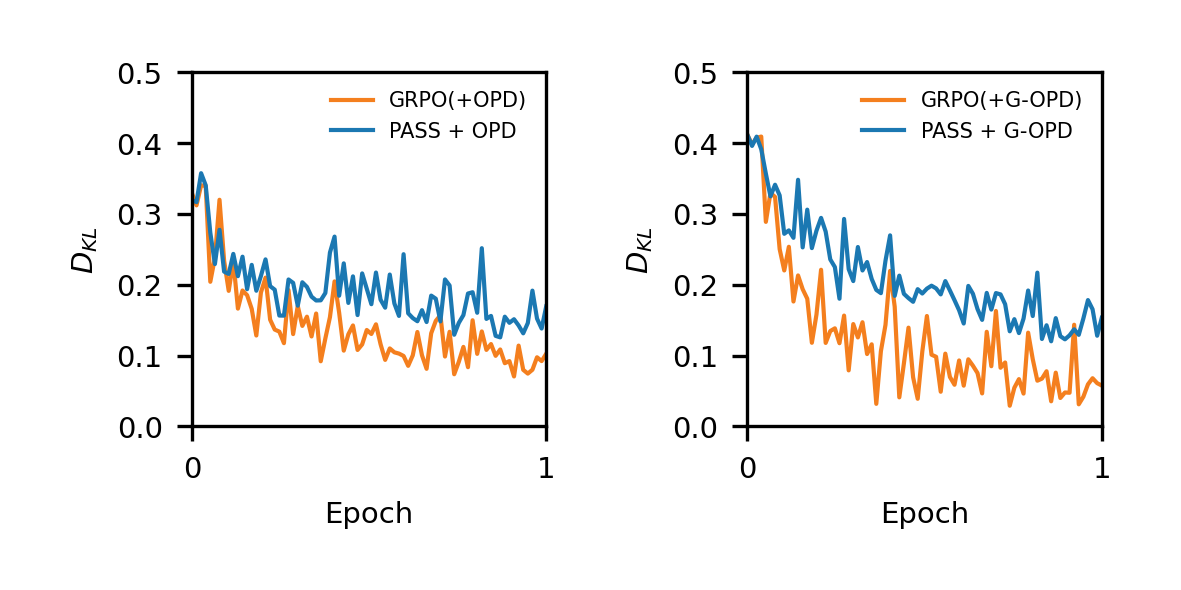}
\caption{Training-time mean token-level KL divergence between the actor and teacher on HotpotQA for the same runs. Left: OPD. Right: G-OPD.}
\label{fig:qa_training_kl}
\end{figure}

Figure~\ref{fig:qa_training_kl} tracks the mean token-level KL divergence between the actor and teacher on the same runs. Both PASS and the controls move closer to the teacher during training, but PASS maintains a higher residual teacher KL. The combination of lower entropy and higher residual KL is consistent with a policy that becomes more confident without simply matching the teacher distribution token by token. In conjunction with the length dynamics, these observations suggest that PASS changes the trajectory-level behavior induced by the teacher-derived signal rather than merely increasing agreement with the teacher.

Taken together, the three diagnostics reveal a coherent training-time pattern across both OPD and G-OPD: PASS moves toward longer responses, lower actor entropy, and a non-vanishing distance from the teacher. These curves are descriptive diagnostics of the resulting policies rather than isolated evidence for any individual PASS operator. In particular, because the comparison from AF-only to PASS jointly introduces shared chunking, RTG, and residual normalization, the observed dynamics should be interpreted as properties of the complete advantage-construction recipe.

\FloatBarrier
\section{Length Performance Analysis of PASS in the OPD Scenario}
\label{app:length}
\label{sec:exp_len_acc}

Table~\ref{tab:length_matrix} reports per-benchmark response lengths for the QA configurations. It complements the accuracy results by showing how each estimator allocates output tokens under the two normalizers.

\begin{table}[h]
\centering
\caption{Average response length in tokens for the QA configurations. Relative changes use the corresponding AF-only control. The outcome-only lengths are not matched to PASS. The GR$^3$-style row denotes the adapted advantage-rescaling control.}
\label{tab:length_matrix}
\resizebox{\textwidth}{!}{%
\begin{tabular}{@{}llccccc@{}}
\toprule
\textbf{Operator} & \textbf{Method} & \textbf{MuSiQue} & \textbf{2Wiki} & \textbf{HotpotQA} & \textbf{Average} & \textbf{$\Delta$ rel.} \\
\midrule
-- & \ding{182} GRPO (outcome-only) & $227 \pm 40$ & $183 \pm 36$ & $151 \pm 26$ & $186.8$ & -- \\
\midrule
 & \ding{183} Process-only OPD & $185 \pm 8$ & $159 \pm 10$ & $137 \pm 4$ & $160.1$ & -- \\
 & \ding{184} AF-only OPD & $188 \pm 6$ & $163 \pm 8$ & $138 \pm 3$ & $163.2$ & -- \\
 & comp. GR$^3$-style + OPD ($\alpha=0.33$) & $188 \pm 3$ & $163 \pm 8$ & $138 \pm 1$ & $163.2$ & $+0\%$ \\
\rowcolor{mylightblue}\cellcolor{white}
 & \ding{185} \textbf{PASS + OPD} & $285 \pm 47$ & $231 \pm 36$ & $201 \pm 25$ & $239.0$ & $+46\%$ \\
 & \ding{183} Process-only G-OPD & $155 \pm 4$ & $124 \pm 4$ & $112 \pm 4$ & $130.3$ & -- \\
 & \ding{184} AF-only G-OPD & $158 \pm 6$ & $128 \pm 4$ & $114 \pm 3$ & $133.4$ & -- \\
\rowcolor{mylightblue}\cellcolor{white}\multirow{-7}{*}{Masked-Norm}
 & \ding{185} \textbf{PASS + G-OPD} & $300 \pm 34$ & $245 \pm 24$ & $204 \pm 14$ & $249.7$ & $+87\%$ \\
\midrule
 & \ding{183} Process-only OPD & $188 \pm 6$ & $165 \pm 6$ & $141 \pm 4$ & $164.5$ & -- \\
 & \ding{184} AF-only OPD & $203 \pm 8$ & $179 \pm 9$ & $147 \pm 7$ & $176.3$ & -- \\
\rowcolor{mylightblue}\cellcolor{white}
 & \ding{185} \textbf{PASS + OPD} & $269 \pm 36$ & $241 \pm 31$ & $210 \pm 23$ & $239.8$ & $+36\%$ \\
 & \ding{183} Process-only G-OPD & $155 \pm 8$ & $128 \pm 6$ & $114 \pm 5$ & $132.4$ & -- \\
 & \ding{184} AF-only G-OPD & $175 \pm 8$ & $148 \pm 7$ & $124 \pm 4$ & $149.0$ & -- \\
\rowcolor{mylightblue}\cellcolor{white}\multirow{-6}{*}{Abs-Max Scaling}
 & \ding{185} \textbf{PASS + G-OPD} & $251 \pm 10$ & $215 \pm 6$ & $197 \pm 10$ & $221.0$ & $+48\%$ \\
\bottomrule
\end{tabular}
}
\end{table}

\paragraph{Length and accuracy together.} Under Masked-Norm, PASS averages 239.0 tokens for OPD and 249.7 for G-OPD, compared with 186.8 for outcome-only training, approximately $28\%$ and $34\%$ more. Under Abs-Max, PASS averages 239.8 and 221.0 tokens. Thus the QA accuracy gains are associated with additional output tokens, unlike the approximately $2\%$ aggregate reduction on mathematics. Since the QA comparisons do not match length or inference compute, the contribution of additional reasoning tokens remains part of the observed trade-off.

\paragraph{Alternative post-processing.} The GR$^3$-style control averages approximately 163 tokens at $54.50$ pass@1, close to AF-only in both measures. PASS instead combines higher mean accuracy with longer outputs in this comparison. These results characterize the accuracy--length behavior of the tested estimators; the fixed-score analysis in Appendix~\ref{app:proofs} provides a separate interpretation of residual normalization.

\FloatBarrier
\section{Training and Evaluation Protocols}
\label{app:hparams}

\paragraph{QA training.} We train for one epoch and evaluate the resulting checkpoint after 39 rollout steps. Each rollout batch contains 128 prompts with 16 sampled responses per prompt; each training global batch contains 128 response trajectories. A rollout step therefore includes multiple optimizer batches. The actor is Qwen2.5-7B-Instruct, the teacher is Qwen2.5-32B-Instruct, and the initial actor serves as the reference and base policy.

Common settings are a 6144-token sequence limit, at most 2048 newly generated tokens, sampling temperature 1.0 and top-$p$ 0.9, and learning rate $10^{-6}$ with a constant schedule after 20 warmup steps. We use Adam with $\beta_1=0.9$, $\beta_2=0.999$, $\epsilon=10^{-8}$, and zero weight decay. GRPO uses ratio clipping $\varepsilon_c=0.2$, reference regularization coefficient $10^{-2}$, and zero entropy bonus. The QA reference regularizer is the masked absolute sampled log-probability difference. Teacher log probabilities enter the advantage, without a separate teacher-KL loss term.

\paragraph{QA data and channels.} All configurations use the same fixed 5000-question subset of the HotpotQA distractor training data. Prompts contain labeled context paragraphs and sentences, the question, and an instruction to reason using only the context, separate steps by newlines, and end with \texttt{Final Answer: [Entity]}. We apply the actor's user-chat template and generation prompt. Training correctness is a binary, case-insensitive exact comparison of the extracted answer, with terminal periods stripped; format is a separate binary check for the final-answer marker. Process-only training uses neither of these two channels in its advantage.

\paragraph{QA seeds and variants.} The comparison contains one outcome-only configuration, 12 process-signal configurations ($2$ signals $\times$ $3$ recipes $\times$ $2$ standardizers), and one GR$^3$-style control. Each uses training seeds 42, 1, and 2, giving 42 runs. Training, data-sampling, and generation seeds are changed together. Table~\ref{tab:hparams_diff} summarizes the estimators.

\begin{table}[h]
\centering
\caption{QA estimator configuration. Process-only ignores the outcome and format channels; their configured weights do not enter that branch. The GR$^3$-style adaptation rescales the AF output without CV or RTG.}
\label{tab:hparams_diff}
\resizebox{\textwidth}{!}{%
\begin{tabular}{@{}llcccccccc@{}}
\toprule
\textbf{Rung} & \textbf{Configuration} & \textbf{AF} & \textbf{CV} & \textbf{RTG} & \textbf{DL} & \textbf{Standardizer} & \textbf{$\lambda$} & \textbf{$\eta$} & \textbf{$k$} \\
\midrule
\ding{182} & Outcome-only ($w_p=0$) & on & off & off & off & -- & -- & -- & -- \\
\midrule
\ding{183} & Process-only OPD & off & off & off & off & Masked-Norm & -- & -- & -- \\
\ding{184} & AF-only OPD & on & off & off & off & Masked-Norm & -- & -- & -- \\
\rowcolor{mylightblue}
\ding{185} & PASS + OPD & on & on & on & on & Masked-Norm & -- & $0.005$ & $0.7$ \\
\ding{183} & Process-only G-OPD & off & off & off & off & Masked-Norm & $1.25$ & -- & -- \\
\ding{184} & AF-only G-OPD & on & off & off & off & Masked-Norm & $1.25$ & -- & -- \\
\rowcolor{mylightblue}
\ding{185} & PASS + G-OPD & on & on & on & on & Masked-Norm & $1.25$ & $0.005$ & $0.7$ \\
\midrule
\ding{183} & Process-only OPD & off & off & off & off & Abs-Max & -- & -- & -- \\
\ding{184} & AF-only OPD & on & off & off & off & Abs-Max & -- & -- & -- \\
\rowcolor{mylightblue}
\ding{185} & PASS + OPD & on & on & on & on & Abs-Max & -- & $0.005$ & $0.7$ \\
\ding{183} & Process-only G-OPD & off & off & off & off & Abs-Max & $1.25$ & -- & -- \\
\ding{184} & AF-only G-OPD & on & off & off & off & Abs-Max & $1.25$ & -- & -- \\
\rowcolor{mylightblue}
\ding{185} & PASS + G-OPD & on & on & on & on & Abs-Max & $1.25$ & $0.005$ & $0.7$ \\
\midrule
comp. & GR$^3$-style OPD ($\alpha=0.33$) & on & off & off & off & Masked-Norm & -- & -- & -- \\
\bottomrule
\end{tabular}%
}
\end{table}

AF-only, PASS, and the GR$^3$-style control use unit channel weights and the format-dependent fusion rule. Outcome-only sets $w_p=0$ and $W=2$, while process-only directly uses the normalized process trace. Normalization epsilon is $10^{-4}$, distinct from the CV tolerance. The six G-OPD configurations use $\lambda=1.25$ and the initial actor as base; the six OPD configurations and the GR$^3$-style control use the teacher--student sampled log ratio.

\paragraph{GR$^3$-style length rescaling.} We apply $q_i=1/(1+\alpha L_i/(\bar L_g+\epsilon))$ to the AF-fused advantage, where $L_i$ is the masked response length and $\bar L_g$ its same-prompt group mean. We use $\alpha=0.33$, following a reported setting of GR$^3$~\citep{li2026tackling}. The control returns $q_i a_{i,t}$ directly, without CV or RTG. Unlike the original raw-reward rescaling before group normalization, this implementation applies the factor after AF; the two orders need not yield the same advantages.

\paragraph{Mathematical training.} The main PASS experiments use seeds 1, 2, and 42, and ORM uses seeds 1111 and 2222. Evaluation uses the checkpoint at optimizer iteration 1248. Training uses 40315 DeepScaleR examples, rollout batches of 512 prompts with 16 retained responses each, training global batch 512, microbatch 1, sequence limit 4096, and response limit 3072. Sampling uses temperature 1.0, top-$p$ 0.9, and unrestricted top-$k$. We use one data epoch and one update epoch per rollout, ratio clipping 0.2, normalization epsilon $10^{-4}$, reference coefficient $10^{-2}$, learning rate $10^{-6}$ with cosine decay, and 20 warmup iterations.

The PRM scores solution steps separated by double newlines, using the positive-class softmax probability at each step-scoring marker. Scores are mapped back to actor-token segments and supplied separately from binary correctness and boxed-answer format rewards. PASS normalizes these three channels and fuses them with $\gamma=1$ in Equation~\ref{eq:prcvdl_advantage_shaping}. Its process mask excludes terminal EOS when present, and the resulting advantages are aligned with next-token predictions. Outcome-only GRPO normalizes the sum of correctness and format rewards within each prompt group. The main PASS settings use $\eta=10^{-8}$ and $k=1$; the single-seed exponent and PRM-scale studies are described in Appendix~\ref{app:math_sensitivity}.

\paragraph{Mathematical evaluation.} For each problem with $n$ responses and $c$ correct responses, we compute $\widehat{\mathrm{pass@}k}=1-\binom{n-c}{k}/\binom{n}{k}$ and then average equally across problems. PASS seeds 1 and 2 use 128 responses per problem, except seed 2 on GSM8K has 126 completed responses; PASS seed 42 uses 32. ORM uses 64, except seed 1111 on AIME24 and AMC23 has 63 completed responses. The two PASS runs with 128-response evaluation use temperature 1.0, top-$p$ 0.95, a 2048-token generation cap, and the model chat template. We report the unweighted mean and sample standard deviation across training runs, computing the seven-benchmark average within each run before its cross-run standard deviation. Deltas are calculated before rounding. Different response counts give different sampling precision within this estimator.

\paragraph{QA evaluation.} We target 16 responses per question with temperature 1.0 and top-$p$ 0.9, and use SQuAD-style answer normalization. Some runs have incomplete evaluation shards. Per-dataset summaries weight run metrics by the number of completed evaluation responses and report the corresponding population standard deviation. Average columns give the three dataset summaries equal weight. This weighting accounts for unequal observed coverage; the reported metrics use completed responses only.

\FloatBarrier
\begin{revision}
\section{Supplementary Matched Comparisons}
\label{app:revision_results}

The matched process-supervision comparisons use training seeds 1111/2222/3333 for each configuration. For each seed, we first average the benchmark scores equally, then report the mean and sample SD over the three seed-level averages. Per-benchmark SDs are also computed across training seeds. No SD is obtained by averaging benchmark SDs. All reported differences use unrounded scores.

\paragraph{Mathematics protocol.} All six configurations use Qwen2.5-Math-7B, Qwen2.5-Math-PRM-7B, the DeepScaleR training configuration, 16 rollout samples per prompt, training/rollout global batch 512, learning rate $10^{-6}$, rollout temperature 1 and top-$p$ 0.9. Each final run has 1,248 recorded training iterations. The five internal estimators use unit AF channel weights; the DL exponent is 1. Native segmentation uses response-specific PRM step endpoints, while CV uses the masked shared trace with tolerance $10^{-8}$. AF direct does not accumulate RTG. Step-GRPO~\citep{li2026stepgrpo} instead aggregates PRM step scores to a scalar trajectory reward using step-attention (temperature 0.3) and format:accuracy:PRM weights 1:4:3 before group-relative normalization. We implement this reward construction with the common actor, data, and budget specified above, rather than the original paper's experimental setting. Final checkpoints are evaluated with 32 samples per problem on all seven benchmarks, using the common evaluation prompts and scoring protocol. Retokenized output length is reported separately from accuracy; training budgets are shared but realized inference lengths differ.

\paragraph{QA protocol.} All five configurations use Qwen2.5-7B-Instruct and Qwen2.5-32B-Instruct as actor and teacher, respectively, the same fixed 5k HotpotQA training file, one epoch, final checkpoint 39, training/rollout batches of 128, 16 rollout responses per prompt, and learning rate $10^{-6}$. They use OPD, Masked-Norm, unit AF weights, tolerance 0.005 for value chunks, and DL exponent 0.7. Token-granularity variants accumulate the fused trace over valid tokens and, when DL is enabled, divide by the power of the remaining valid-token count. Token-level accumulation uses every valid response token as a credit unit; CV-based accumulation instead uses the retained shared endpoints. Evaluation uses 16 responses per question, temperature 1, top-$p$ 0.9, maximum 2048 new tokens, and SQuAD-normalized final-answer exact match with aliases. The complete evaluation sets contain 7,405 HotpotQA, 2,417 MuSiQue, and 12,576 2Wiki questions. The 15 checkpoints supply 5,375,520 saved responses for retokenization, with no duplicate keys. Each checkpoint's tokenizer is used with no added special tokens. Token counts need not equal sampler token counts after text decoding and re-encoding.

\paragraph{Interpretation.} Native-step versus CV contrasts hold RTG and DL settings fixed but change both endpoint sampling and accumulation units. Their effects interact with DL. The OPD token RTG+DL variant has slightly higher macro pass@1 than CV PASS, whereas PASS has higher pass@8/16 on all three benchmarks, shorter responses, and lower observed pass@1 seed SD. The GR$^3$-style comparison applies advantage rescaling, while the OPD component comparisons vary accumulation and residual normalization. Same-density fixed-length segmentation and inference-length-matched comparisons remain open.

\begin{table}[htbp]
\ifrevisionhighlight\color{blue}\fi
\centering
\caption{Mathematics supplementary pass@1 (\%). Entries are mean and sample SD across three training seeds; the macro average is computed within seed before aggregation.}
\label{tab:revision_math_pass1}
\resizebox{\textwidth}{!}{%
\begin{tabular}{@{}lrrrrrr@{}}
\toprule
\textbf{Benchmark} & \textbf{AF direct} & \textbf{Native PRM-step RTG} & \textbf{Native PRM-step RTG+DL} & \textbf{CV+RTG} & \cellcolor{mylightblue}\textbf{PASS (Ours)} & \textbf{Step-GRPO} \\
\midrule
AIME24 & $3.82 \pm 3.31$ & $8.09 \pm 0.78$ & $15.45 \pm 0.77$ & $15.38 \pm 0.39$ & \cellcolor{mylightblue}$16.84 \pm 1.17$ & $15.73 \pm 1.93$ \\
AIME25 & $1.11 \pm 0.97$ & $1.81 \pm 0.61$ & $6.67 \pm 0.36$ & $8.58 \pm 0.90$ & \cellcolor{mylightblue}$9.51 \pm 0.30$ & $4.97 \pm 0.98$ \\
AMC23 & $28.83 \pm 11.96$ & $35.05 \pm 0.91$ & $51.51 \pm 1.70$ & $51.67 \pm 1.07$ & \cellcolor{mylightblue}$55.86 \pm 1.09$ & $48.88 \pm 2.88$ \\
GSM8K & $42.43 \pm 15.59$ & $44.67 \pm 1.78$ & $71.99 \pm 1.43$ & $69.36 \pm 1.01$ & \cellcolor{mylightblue}$76.91 \pm 1.72$ & $64.60 \pm 6.08$ \\
MATH & $33.94 \pm 9.90$ & $34.93 \pm 1.27$ & $51.58 \pm 0.72$ & $50.55 \pm 0.60$ & \cellcolor{mylightblue}$55.44 \pm 1.00$ & $49.46 \pm 2.88$ \\
Minerva & $16.07 \pm 4.10$ & $16.05 \pm 0.52$ & $26.54 \pm 0.49$ & $26.21 \pm 0.40$ & \cellcolor{mylightblue}$28.39 \pm 0.93$ & $17.48 \pm 3.90$ \\
Olympiad & $4.90 \pm 2.22$ & $5.65 \pm 0.36$ & $10.33 \pm 0.26$ & $10.76 \pm 0.28$ & \cellcolor{mylightblue}$11.78 \pm 0.34$ & $8.90 \pm 1.65$ \\
\midrule
\textbf{Average} & $18.73 \pm 6.84$ & $20.89 \pm 0.76$ & $33.44 \pm 0.66$ & $33.22 \pm 0.41$ & \cellcolor{mylightblue}$36.39 \pm 0.54$ & $30.00 \pm 2.66$ \\
\bottomrule
\end{tabular}%
}
\end{table}

\begin{table}[htbp]
\ifrevisionhighlight\color{blue}\fi
\centering
\caption{Mathematics supplementary pass@8 (\%). Entries are mean and sample SD across three training seeds; the macro average is computed within seed before aggregation.}
\label{tab:revision_math_pass8}
\resizebox{\textwidth}{!}{%
\begin{tabular}{@{}lrrrrrr@{}}
\toprule
\textbf{Benchmark} & \textbf{AF direct} & \textbf{Native PRM-step RTG} & \textbf{Native PRM-step RTG+DL} & \textbf{CV+RTG} & \cellcolor{mylightblue}\textbf{PASS (Ours)} & \textbf{Step-GRPO} \\
\midrule
AIME24 & $14.96 \pm 12.25$ & $24.28 \pm 2.04$ & $40.92 \pm 1.72$ & $43.90 \pm 2.07$ & \cellcolor{mylightblue}$42.87 \pm 1.42$ & $41.64 \pm 3.51$ \\
AIME25 & $5.75 \pm 4.99$ & $9.75 \pm 1.66$ & $21.63 \pm 3.39$ & $27.21 \pm 0.32$ & \cellcolor{mylightblue}$26.95 \pm 0.79$ & $20.08 \pm 2.70$ \\
AMC23 & $62.56 \pm 21.64$ & $76.98 \pm 0.93$ & $83.74 \pm 1.36$ & $83.54 \pm 0.32$ & \cellcolor{mylightblue}$83.98 \pm 2.38$ & $81.72 \pm 1.52$ \\
GSM8K & $82.24 \pm 15.64$ & $83.69 \pm 2.69$ & $94.27 \pm 0.33$ & $94.30 \pm 0.09$ & \cellcolor{mylightblue}$95.01 \pm 0.18$ & $94.59 \pm 0.53$ \\
MATH & $57.49 \pm 14.73$ & $62.37 \pm 1.01$ & $70.35 \pm 0.23$ & $71.23 \pm 0.07$ & \cellcolor{mylightblue}$71.60 \pm 0.10$ & $70.58 \pm 0.52$ \\
Minerva & $40.25 \pm 6.94$ & $39.75 \pm 1.70$ & $47.17 \pm 0.33$ & $47.39 \pm 0.26$ & \cellcolor{mylightblue}$48.69 \pm 0.03$ & $40.98 \pm 4.02$ \\
Olympiad & $12.47 \pm 5.33$ & $15.44 \pm 0.11$ & $20.07 \pm 0.60$ & $20.73 \pm 0.11$ & \cellcolor{mylightblue}$21.10 \pm 0.28$ & $19.13 \pm 1.00$ \\
\midrule
\textbf{Average} & $39.39 \pm 11.63$ & $44.61 \pm 0.94$ & $54.02 \pm 0.84$ & $55.47 \pm 0.34$ & \cellcolor{mylightblue}$55.74 \pm 0.43$ & $52.67 \pm 1.42$ \\
\bottomrule
\end{tabular}%
}
\end{table}

\begin{table}[htbp]
\ifrevisionhighlight\color{blue}\fi
\centering
\caption{Mathematics supplementary pass@16 (\%). Entries are mean and sample SD across three training seeds; the macro average is computed within seed before aggregation.}
\label{tab:revision_math_pass16}
\resizebox{\textwidth}{!}{%
\begin{tabular}{@{}lrrrrrr@{}}
\toprule
\textbf{Benchmark} & \textbf{AF direct} & \textbf{Native PRM-step RTG} & \textbf{Native PRM-step RTG+DL} & \textbf{CV+RTG} & \cellcolor{mylightblue}\textbf{PASS (Ours)} & \textbf{Step-GRPO} \\
\midrule
AIME24 & $20.71 \pm 16.92$ & $31.50 \pm 2.66$ & $48.80 \pm 2.72$ & $51.10 \pm 2.15$ & \cellcolor{mylightblue}$48.42 \pm 2.89$ & $47.96 \pm 5.39$ \\
AIME25 & $8.47 \pm 7.42$ & $14.08 \pm 1.54$ & $27.95 \pm 6.32$ & $33.19 \pm 0.96$ & \cellcolor{mylightblue}$31.45 \pm 1.86$ & $27.07 \pm 3.90$ \\
AMC23 & $70.34 \pm 21.13$ & $83.50 \pm 2.17$ & $87.86 \pm 1.28$ & $88.34 \pm 1.53$ & \cellcolor{mylightblue}$89.24 \pm 3.06$ & $86.58 \pm 1.85$ \\
GSM8K & $88.95 \pm 10.41$ & $90.37 \pm 1.79$ & $95.71 \pm 0.35$ & $95.82 \pm 0.10$ & \cellcolor{mylightblue}$96.13 \pm 0.17$ & $96.26 \pm 0.08$ \\
MATH & $62.57 \pm 13.86$ & $67.92 \pm 0.67$ & $73.06 \pm 0.19$ & $73.75 \pm 0.05$ & \cellcolor{mylightblue}$73.82 \pm 0.08$ & $73.54 \pm 0.16$ \\
Minerva & $46.94 \pm 6.23$ & $46.61 \pm 2.39$ & $51.99 \pm 0.61$ & $51.51 \pm 0.56$ & \cellcolor{mylightblue}$53.12 \pm 0.12$ & $47.52 \pm 2.98$ \\
Olympiad & $14.97 \pm 6.03$ & $18.40 \pm 0.21$ & $22.44 \pm 0.66$ & $22.88 \pm 0.12$ & \cellcolor{mylightblue}$23.29 \pm 0.23$ & $22.02 \pm 0.69$ \\
\midrule
\textbf{Average} & $44.71 \pm 11.63$ & $50.34 \pm 0.56$ & $58.26 \pm 1.42$ & $59.51 \pm 0.59$ & \cellcolor{mylightblue}$59.35 \pm 0.51$ & $57.28 \pm 1.03$ \\
\bottomrule
\end{tabular}%
}
\end{table}

\begin{table}[htbp]
\ifrevisionhighlight\color{blue}\fi
\centering
\caption{Mathematics supplementary average response length in tokens (retokenized). Entries are mean and sample SD across three training seeds; the macro average is computed within seed before aggregation.}
\label{tab:revision_math_tokens}
\resizebox{\textwidth}{!}{%
\begin{tabular}{@{}lrrrrrr@{}}
\toprule
\textbf{Benchmark} & \textbf{AF direct} & \textbf{Native PRM-step RTG} & \textbf{Native PRM-step RTG+DL} & \textbf{CV+RTG} & \cellcolor{mylightblue}\textbf{PASS (Ours)} & \textbf{Step-GRPO} \\
\midrule
AIME24 & $498.8 \pm 507.4$ & $996.1 \pm 38.4$ & $1164.3 \pm 7.8$ & $1192.0 \pm 8.6$ & \cellcolor{mylightblue}$1103.9 \pm 16.8$ & $599.2 \pm 174.0$ \\
AIME25 & $466.6 \pm 497.6$ & $1023.8 \pm 39.8$ & $1200.2 \pm 29.0$ & $1251.2 \pm 11.5$ & \cellcolor{mylightblue}$1195.4 \pm 22.5$ & $550.1 \pm 163.3$ \\
AMC23 & $481.5 \pm 444.5$ & $854.4 \pm 8.9$ & $972.8 \pm 9.4$ & $1036.6 \pm 14.5$ & \cellcolor{mylightblue}$943.0 \pm 15.9$ & $548.2 \pm 90.3$ \\
GSM8K & $307.0 \pm 224.5$ & $489.3 \pm 6.9$ & $525.4 \pm 13.4$ & $602.6 \pm 13.7$ & \cellcolor{mylightblue}$482.0 \pm 24.9$ & $243.0 \pm 16.5$ \\
MATH & $370.9 \pm 314.5$ & $618.0 \pm 4.7$ & $695.1 \pm 8.3$ & $809.2 \pm 11.0$ & \cellcolor{mylightblue}$684.9 \pm 20.0$ & $384.1 \pm 44.3$ \\
Minerva & $391.4 \pm 239.2$ & $634.5 \pm 18.7$ & $749.2 \pm 1.7$ & $891.2 \pm 11.0$ & \cellcolor{mylightblue}$743.6 \pm 20.3$ & $279.2 \pm 70.8$ \\
Olympiad & $416.4 \pm 386.2$ & $774.9 \pm 15.6$ & $897.0 \pm 9.3$ & $998.6 \pm 8.5$ & \cellcolor{mylightblue}$893.2 \pm 16.9$ & $447.1 \pm 125.7$ \\
\midrule
\textbf{Average} & $419.0 \pm 371.9$ & $770.1 \pm 17.5$ & $886.3 \pm 8.5$ & $968.8 \pm 6.8$ & \cellcolor{mylightblue}$863.7 \pm 17.0$ & $435.9 \pm 97.0$ \\
\bottomrule
\end{tabular}%
}
\end{table}

\begin{table}[htbp]
\ifrevisionhighlight\color{blue}\fi
\centering
\caption{OPD QA supplementary pass@1 (\%). Entries are mean and sample SD across three training seeds; the macro average is computed within seed before aggregation.}
\label{tab:revision_qa_pass1}
\begin{tabular}{@{}lrrrr@{}}
\toprule
\textbf{Method} & \textbf{HotpotQA} & \textbf{MuSiQue} & \textbf{2Wiki} & \textbf{Average} \\
\midrule
AF+CV & $61.57 \pm 0.21$ & $38.59 \pm 0.45$ & $63.65 \pm 0.03$ & $54.60 \pm 0.22$ \\
AF+CV+RTG & $62.25 \pm 0.11$ & $44.99 \pm 0.32$ & $65.45 \pm 0.46$ & $57.57 \pm 0.29$ \\
AF+token RTG & $62.25 \pm 0.29$ & $44.19 \pm 0.34$ & $65.43 \pm 0.27$ & $57.29 \pm 0.12$ \\
\rowcolor{mylightblue}
\textbf{PASS (Ours)} & $63.09 \pm 0.09$ & $45.49 \pm 0.56$ & $66.06 \pm 0.28$ & $58.22 \pm 0.22$ \\
AF+token RTG+token DL & $63.37 \pm 0.36$ & $46.19 \pm 1.05$ & $65.68 \pm 0.70$ & $58.41 \pm 0.56$ \\
\bottomrule
\end{tabular}
\end{table}

\begin{table}[htbp]
\ifrevisionhighlight\color{blue}\fi
\centering
\caption{OPD QA supplementary pass@8 (\%). Entries are mean and sample SD across three training seeds; the macro average is computed within seed before aggregation.}
\label{tab:revision_qa_pass8}
\begin{tabular}{@{}lrrrr@{}}
\toprule
\textbf{Method} & \textbf{HotpotQA} & \textbf{MuSiQue} & \textbf{2Wiki} & \textbf{Average} \\
\midrule
AF+CV & $72.35 \pm 0.08$ & $61.16 \pm 0.38$ & $74.93 \pm 0.31$ & $69.48 \pm 0.20$ \\
AF+CV+RTG & $72.98 \pm 0.12$ & $66.02 \pm 0.53$ & $77.03 \pm 0.62$ & $72.01 \pm 0.41$ \\
AF+token RTG & $71.28 \pm 0.51$ & $63.59 \pm 1.84$ & $75.68 \pm 0.73$ & $70.19 \pm 0.99$ \\
\rowcolor{mylightblue}
\textbf{PASS (Ours)} & $72.83 \pm 0.12$ & $66.53 \pm 0.18$ & $76.84 \pm 0.26$ & $72.07 \pm 0.13$ \\
AF+token RTG+token DL & $72.02 \pm 0.53$ & $65.30 \pm 0.13$ & $75.79 \pm 0.40$ & $71.04 \pm 0.32$ \\
\bottomrule
\end{tabular}
\end{table}

\begin{table}[htbp]
\ifrevisionhighlight\color{blue}\fi
\centering
\caption{OPD QA supplementary pass@16 (\%). Entries are mean and sample SD across three training seeds; the macro average is computed within seed before aggregation.}
\label{tab:revision_qa_pass16}
\begin{tabular}{@{}lrrrr@{}}
\toprule
\textbf{Method} & \textbf{HotpotQA} & \textbf{MuSiQue} & \textbf{2Wiki} & \textbf{Average} \\
\midrule
AF+CV & $74.51 \pm 0.13$ & $66.16 \pm 0.42$ & $77.12 \pm 0.37$ & $72.60 \pm 0.21$ \\
AF+CV+RTG & $75.05 \pm 0.14$ & $70.61 \pm 0.48$ & $79.16 \pm 0.71$ & $74.94 \pm 0.42$ \\
AF+token RTG & $73.09 \pm 0.56$ & $67.67 \pm 2.30$ & $77.62 \pm 0.79$ & $72.80 \pm 1.18$ \\
\rowcolor{mylightblue}
\textbf{PASS (Ours)} & $74.83 \pm 0.13$ & $71.09 \pm 0.28$ & $79.04 \pm 0.33$ & $74.99 \pm 0.10$ \\
AF+token RTG+token DL & $73.83 \pm 0.63$ & $69.48 \pm 0.34$ & $77.80 \pm 0.37$ & $73.70 \pm 0.40$ \\
\bottomrule
\end{tabular}
\end{table}

\begin{table}[htbp]
\ifrevisionhighlight\color{blue}\fi
\centering
\caption{OPD QA supplementary average response length in tokens (retokenized). Entries are mean and sample SD across three training seeds; the macro average is computed within seed before aggregation.}
\label{tab:revision_qa_tokens}
\begin{tabular}{@{}lrrrr@{}}
\toprule
\textbf{Method} & \textbf{HotpotQA} & \textbf{MuSiQue} & \textbf{2Wiki} & \textbf{Average} \\
\midrule
AF+CV & $131.2 \pm 6.9$ & $183.9 \pm 8.0$ & $150.6 \pm 9.4$ & $155.2 \pm 8.0$ \\
AF+CV+RTG & $262.5 \pm 8.5$ & $425.0 \pm 7.9$ & $300.9 \pm 12.0$ & $329.5 \pm 9.3$ \\
AF+token RTG & $264.4 \pm 24.7$ & $413.4 \pm 81.5$ & $340.8 \pm 40.4$ & $339.5 \pm 41.0$ \\
\rowcolor{mylightblue}
\textbf{PASS (Ours)} & $227.4 \pm 9.5$ & $334.9 \pm 31.9$ & $280.6 \pm 13.2$ & $281.0 \pm 16.4$ \\
AF+token RTG+token DL & $257.7 \pm 56.0$ & $424.6 \pm 80.2$ & $322.6 \pm 75.6$ & $335.0 \pm 70.3$ \\
\bottomrule
\end{tabular}
\end{table}

\FloatBarrier
\section{Training Dynamics and Shared-Endpoint Diagnostics}
\label{app:revision_dynamics}

The CV+RTG and PASS logs each contain three seeds with records for iterations 1--1,248. Table~\ref{tab:revision_diagnostics} averages each logged statistic within seed before computing the three-seed mean and sample SD. Quantiles and coverage rates are averages of step-level statistics, not pooled token statistics. Training lengths differ from evaluation retokenizations.

\begin{table}[htbp]
\ifrevisionhighlight\color{blue}\fi
\centering
\caption{Mathematics training diagnostics from all recorded steps, three seeds per configuration. Shared endpoints are those actually retained for advantage construction.}
\label{tab:revision_diagnostics}
\begin{tabular}{@{}lrr@{}}
\toprule
\textbf{Metric} & \textbf{CV+RTG} & \cellcolor{mylightblue}\textbf{PASS (Ours)} \\
\midrule
Training response tokens & $1156.31\pm15.18$ & \cellcolor{mylightblue}$943.91\pm23.90$ \\
Effective endpoints per response & $155.88\pm1.33$ & \cellcolor{mylightblue}$131.77\pm1.83$ \\
Chunk length mean & $7.59\pm0.08$ & \cellcolor{mylightblue}$7.25\pm0.11$ \\
Chunk length P50 & $4.30\pm0.02$ & \cellcolor{mylightblue}$4.25\pm0.03$ \\
Chunk length P90 & $16.34\pm0.24$ & \cellcolor{mylightblue}$15.23\pm0.28$ \\
Uncovered valid-tail tokens (\%) & $0.0425\pm0.0084$ & \cellcolor{mylightblue}$0.1544\pm0.0079$ \\
All zero-advantage valid tokens (\%) & $0.0433\pm0.0083$ & \cellcolor{mylightblue}$0.1555\pm0.0079$ \\
\bottomrule
\end{tabular}
\end{table}

\begin{figure}[htbp]
\ifrevisionhighlight\color{blue}\fi
\centering
\includegraphics[width=0.94\textwidth]{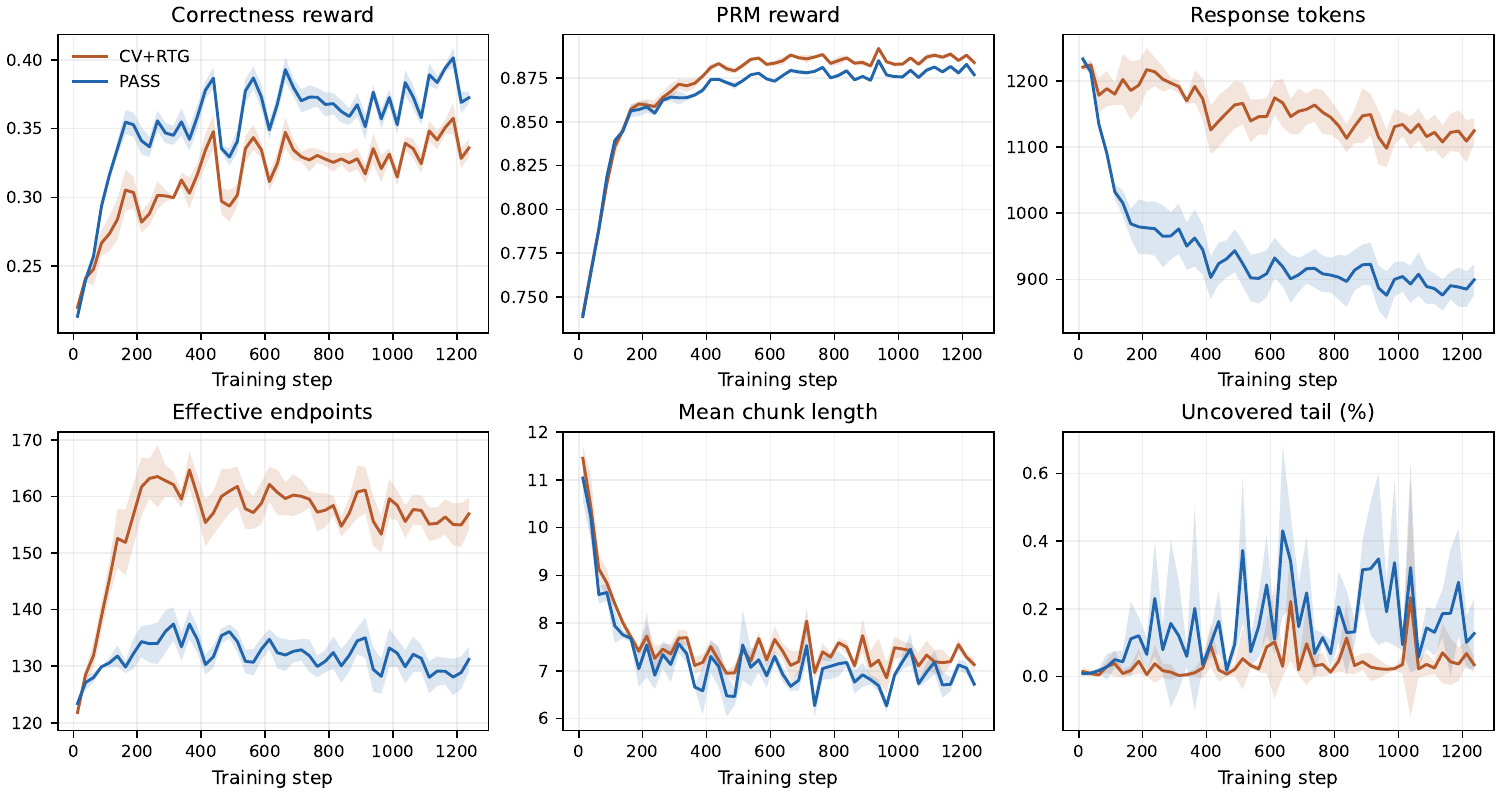}
\caption{Mathematics training dynamics: correctness reward, PRM reward, response length, effective endpoint count, mean chunk length, and uncovered-tail fraction. Curves are means and shaded bands are sample SDs over three seeds after averaging within consecutive 25-step bins (including the final partial bin). The horizontal coordinate is each bin's mean iteration. Rewards are the logged rollout correctness and PRM scores, not the AF-fused advantage. These descriptive curves do not by themselves establish causal mechanisms or general convergence guarantees.}
\label{fig:revision_dynamics}
\end{figure}

\paragraph{Coverage semantics.} Actual RTG/broadcast retains shared endpoints only. A per-response terminal endpoint is appended solely for depth-wise length diagnostics, after tail coverage is measured. Beyond the last retained endpoint, valid tokens have zero direct advantage; with no endpoint, all advantages on this path are zero. The PRM mask excludes terminal EOS. Reference regularization is separate. Covered tokens can also have zero advantage, so total zero-advantage and tail fractions differ.

\paragraph{Supplementary native-boundary diagnostic.} A separate PASS run with training seed 4444 records native PRM endpoints over all 1,248 training steps. For each response, let $N_i$ be its valid native PRM endpoint set and $B_i$ its retained shared CV endpoint set. We measure native-boundary recall as $|N_i\cap B_i|/|N_i|$ and the native-boundary fraction among CV endpoints as $|N_i\cap B_i|/|B_i|$, averaging each ratio over responses with a nonzero corresponding denominator within each training step, then equally over recorded steps. These are means of response-level ratios, not ratios of globally pooled counts. Native-boundary recall is $82.13\%$ and the reverse fraction is $8.85\%$. This single-seed diagnostic uses native PRM endpoints, distinct from the individual reward-value transitions underlying the three-seed overlap statistic below; the runs and statistics are not pooled. It suggests that shared credit units can preserve much of the native scoring structure while adding accumulation endpoints. It measures neither semantic alignment nor batch-resampling stability, and additional CV endpoints are not errors relative to an objective of reconstructing native segmentation.

\paragraph{Boundary limitations.} Shared positions need not align reasoning stages. Individual jumps can cancel, and changes in the active-response set can create boundaries. Logged shared-endpoint matches to individual reward-value transitions average $10.61\pm0.01\%$ for CV+RTG and $9.97\pm0.08\%$ for PASS (the same time-then-seed aggregation), not semantic or bidirectional native-PRM alignment. Final-run batch-resampling stability, active-set-only boundary rates, no-endpoint response rates, and length-stratified coverage remain unmeasured. Native-step/token-level controls test estimator utility rather than semantic segmentation correctness.
\end{revision}

\end{document}